\documentclass[letterpaper]{article} %
\usepackage{aaai2026}  %

\nocopyright
\usepackage{times}  %
\usepackage{helvet}  %
\usepackage{courier}  %
\usepackage[hyphens]{url}  %
\usepackage{graphicx} %
\usepackage{natbib}  %
\usepackage{caption} %
\usepackage{algorithm}
\usepackage[noend]{algpseudocode}

\usepackage{newfloat}
\usepackage{listings}
\DeclareCaptionStyle{ruled}{labelfont=normalfont,labelsep=colon,strut=off} %
\floatstyle{ruled}
\newfloat{listing}{tb}{lst}{}
\floatname{listing}{Listing}
\usepackage{microtype}

\usepackage{latexsym}
\usepackage{amssymb}
\usepackage{amsmath}
\usepackage{amsthm}
\usepackage{mathtools}

\usepackage[dvipsnames]{xcolor}
\usepackage{booktabs}
\usepackage[inline]{enumitem}
\usepackage{graphicx}
\usepackage{color}
\usepackage[all]{foreign}
\usepackage[colorinlistoftodos]{todonotes}
\usepackage{tikz}
\usetikzlibrary{patterns,arrows,arrows.meta,calc,shapes,shadows,decorations.pathmorphing,decorations.pathreplacing,automata,shapes.multipart,positioning,shapes.geometric,fit,circuits,trees,shapes.gates.logic.US,fit,backgrounds,decorations.text,topaths}
\usepackage{pgfplots}
\pgfplotsset{compat=1.9}
\usepgfplotslibrary{fillbetween}

\usepackage{cleveref}
\usepackage{listings}

\usepackage{thmtools} 
\usepackage{thm-restate}
\usepackage{footmisc}

\usepackage{tabularx}
\usepackage{booktabs}
\usepackage{scrextend}
\usepackage{threeparttable}
\usepackage{multicol}
\usepackage{multirow}

\crefname{equation}{Eq.}{Eqs.}
\crefname{pluralequation}{Eqs.}{Eqs.}

\crefname{algorithm}{Algorithm}{Algorithm}

\crefname{figure}{Fig.}{Figs.}
\crefname{pluralfigure}{Figs.}{Figs.}

\crefname{section}{Sect.}{Sects.}
\crefname{pluralsection}{Sects.}{Sects.}

\crefname{table}{Table}{Table}
\crefname{pluraltable}{Tables}{Tables}

\crefname{definition}{Def.}{Def.}
\crefname{pluraldefinition}{Defs.}{Defs.}

\crefname{theorem}{Theorem}{Theorems}
\crefname{pluraltheorem}{Theorems}{Theorems}

\crefname{corollary}{Corollary}{Corollaries}
\crefname{pluralcorollary}{Corollaries}{Corollaries}

\crefname{lemma}{Lemma}{Lemmas}
\crefname{plurallemma}{Lemmas}{Lemmas}

\crefname{example}{Example}{Example}
\crefname{pluralexample}{Examples}{Examples}

\crefname{problem}{Problem}{Problem}
\crefname{pluralproblem}{Problems}{Problems}

\crefname{assumption}{Assumption}{Assumption}
\crefname{pluralassumption}{Assumptions}{Assumptions}

\crefname{remark}{Remark}{Remark}
\crefname{pluralremark}{Remarks}{Remarks}

\crefname{appendix}{Appendix}{Appendices}
\crefname{pluralappendix}{Appendices}{Appendices}

\crefname{line}{Line}{Line}
\crefname{pluralline}{Lines}{Lines}

\newtheorem{example}{Example}
\newtheorem{problem}{Problem}
\newtheorem{theorem}{Theorem}

\newtheorem{definition}{Definition}
\newtheorem{remark}{Remark}
\newtheorem{assumption}{Assumption}

\newcommand*{\States}{\ensuremath{S}}
\newcommand*{\Actions}{\ensuremath{Act}}
\newcommand*{\initState}{\ensuremath{s_I}}

\newcommand*{\MDP}{\ensuremath{(S,\initState,A,P,R,\gamma)}}

\newcommand*{\rmdp}{\ensuremath{\mathcal{M}_R}}
\newcommand*{\RMDP}{\ensuremath{(S,\initState,A,\cP,R,\gamma)}}

\newcommand*{\policy}{\ensuremath{\pi}}

\newcommand*{\policySpace}{\ensuremath{\Pi}}

\newcommand*{\RR}{\mathbb{R}}
\newcommand*{\RRnn}{\RR_{\geq0}}

\newcommand{\bE}{\mathbb{E}}

\newcommand{\cP}{\mathcal{P}}
\newcommand{\cT}{\mathcal{T}}

\newcommand{\ba}{\mathbf{a}}
\newcommand{\bb}{\mathbf{b}}
\newcommand{\bc}{\mathbf{c}}

\newcommand*{\dist}[1]{\Delta_{#1}}

\newcommand{\relinterior}[1]{\ensuremath{\mathrm{relint(#1)}}}
\newcommand{\diag}[1]{\ensuremath{\mathrm{diag(#1)}}}

\newcommand*{\dotp}[2]{\langle #1, #2 \rangle}

\DeclareMathOperator*{\argmax}{argmax}
\DeclareMathOperator*{\argmin}{argmin}

\usepackage{xspace}

\newcommand{\BE}{\ensuremath{\mathrm{BE}}}
\newcommand{\RBE}{\ensuremath{\mathrm{ORBE}}\xspace}

\usepackage{xspace}

\newcommand{\besteffort}{best-effort\xspace}
\newcommand{\Besteffort}{Best-effort\xspace}
\newcommand{\BestEffort}{Best-Effort\xspace}  %

\newcommand{\optimalrobust}{optimal robust\xspace}

\newcommand{\OptimalRobust}{Optimal Robust\xspace}  %

\newcommand{\srectangularInner}{$s$-rectangular}
\newcommand{\srectangular}{\srectangularInner\xspace}
\newcommand{\srectangularity}{\srectangularInner ity\xspace}

\definecolor{pastelblue}{rgb}{0.7,0.7,1}
\definecolor{pastelgreen}{rgb}{0.6,1,0.6}
\definecolor{pastelpink}{rgb}{1,0.6,0.6}

\definecolor{red}{rgb}{0.745,0.192,0.102}
\definecolor{darkgreen}{RGB}{34,161,55}
\definecolor{darkblue}{RGB}{0,33,71}
\definecolor{ruhuisstijlrood}{rgb}{0.745,0.192,0.102}
\definecolor{ruhuisstijlzwart}{rgb}{0,0,0}
\definecolor{ruhuisstijlwit}{rgb}{0.98,0.98,0.98}

\definecolor{darkred}{rgb}{0.560784314,0.125490196,0.0666666667}

\definecolor{color1}{RGB}{55,126,184} %
\definecolor{color2}{RGB}{228,26,28} %
\definecolor{color3}{RGB}{77,175,74} %
\definecolor{color4}{RGB}{152,78,163} %
\definecolor{color5}{RGB}{255,127,0} %
\definecolor{color6}{rgb}{0.5, 1.0, 0.83} %
\definecolor{color7}{rgb}{1.0, 0.0, 1.0} %
\definecolor{color8}{rgb}{0.66, 0.66, 0.66} %

\colorlet{plot1}{color1}
\colorlet{plot2}{color2}
\colorlet{plot3}{color3}
\colorlet{plot4}{color4}
\colorlet{plot5}{color5}

\title{Best-Effort Policies for Robust Markov Decision Processes %
}
\author{
    Alessandro Abate%
    ,
    Thom Badings%
    ,
    Giuseppe De Giacomo%
    ,
    Francesco Fabiano%
}
\affiliations{
    Department of Computer Science, University of Oxford, Oxford, United Kingdom\\
    \{alessandro.abate, thom.badings, giuseppe.degiacomo, francesco.fabiano\}@cs.ox.ac.uk
}

\usepackage{bibentry}

\begin{document}

\maketitle

\begin{abstract}
We study the common generalization of Markov decision processes (MDPs) with sets of transition probabilities, known as robust MDPs (RMDPs). A standard goal in RMDPs is to compute a policy that maximizes the expected return under an adversarial choice of the transition probabilities. If the uncertainty in the probabilities is independent between the states, known as \srectangularity, such optimal robust policies can be computed efficiently using robust value iteration. However, there might still be multiple optimal robust policies, which, while equivalent with respect to the worst-case, reflect different expected returns under non-adversarial choices of the transition probabilities. Hence, we propose a refined policy selection criterion for RMDPs, drawing inspiration from the notions of \emph{dominance} and \emph{best-effort} in game theory. Instead of seeking a policy that only maximizes the worst-case expected return, we additionally require the policy to achieve a \emph{maximal} expected return under different (\ie not fully adversarial) transition probabilities. We call such a policy an \emph{\optimalrobust \besteffort} (\RBE) policy. We prove that \RBE policies always exist, characterize their structure, and present an algorithm to compute them with a manageable overhead compared to standard robust value iteration. \RBE policies offer a principled tie-breaker among optimal robust policies. Numerical experiments show the feasibility of our approach.
\end{abstract}

\section{Introduction}
\label{sec:intro}

\emph{Markov decision processes} (MDPs) are the standard model for sequential decision making in stochastic environments and are ubiquitous in artificial intelligence (AI)~\cite{DBLP:books/daglib/0023820}, operations research~\cite{davis2018markov}, control theory~\cite{aastrom2012introduction}, and robotics~\cite{DBLP:journals/ai/HanheideGHPSAJG17}.
Within AI, MDPs are at the core of many model-based reinforcement learning methods~\cite{DBLP:journals/ftml/MoerlandBPJ23}.
Solving an MDP amounts to computing a \emph{policy} (or \emph{strategy}) for the agent, \ie a mapping from states to actions, that maximizes a particular performance value, such as the expected (discounted) cumulative reward~\cite{DBLP:books/wi/Puterman94}.

\paragraph{Robust MDPs.}
A fundamental limitation of MDPs is the requirement to specify transition probabilities precisely.
In practice, accurately determining these probabilities can be challenging, especially if parameters are uncertain or if the model is learned from data~\cite{DBLP:journals/sttt/BadingsSSJ23}.
Moreover, optimal policies may be sensitive to small changes in the transition probabilities~\cite{DBLP:conf/icml/MannorSST04}.
To address this issue, \emph{robust MDPs} (RMDPs) generalize MDPs by allowing for \emph{sets of transition probabilities}~\cite{DBLP:journals/mor/Iyengar05,DBLP:journals/ior/NilimG05,DBLP:journals/mor/WiesemannKR13}.
That is, instead of assigning precise probabilities between 0 and 1, the transitions in an RMDP are described by a set of feasible probabilities, called the \emph{uncertainty set} of the RMDP.

The standard objective in an RMDP is to compute an \emph{\optimalrobust policy}, defined as a policy that \emph{maximizes} the expected return under the \emph{minimizing} (\ie worst-case) transition probabilities in the uncertainty set. %
Unfortunately, computing \optimalrobust policies under general uncertainty sets is NP-hard~\cite{DBLP:journals/mor/WiesemannKR13}.
To ensure tractability, uncertainty sets are commonly assumed to be convex as well as independent between the states and/or actions of the RMDP, referred to as \emph{rectangularity} of the uncertainty set.
Under these assumptions, \optimalrobust policies can be computed, \eg using robust value iteration.

\paragraph{The adversarial nature of RMDPs.}
When computing an \optimalrobust policy, the choice of transition probabilities is inherently adversarial.
However, in many scenarios, the choice of transition probabilities is \emph{not} actively working against the agent, making this assumption overly conservative.
Take, for example, an autonomous drone flying through uncertain wind conditions.
Clearly, the wind conditions do not depend on the drone's control policy, so reasoning solely about the worst-case conditions might be too conservative.
Moreover, multiple \optimalrobust policies may exist, even though their performance under non-adversarial conditions may differ.
We thus raise the vital question: can we compute a policy that is optimal in the worst case, but also ``is best'' when the environment does not act fully adversarially?

\paragraph{\Besteffort policies.}
To address the limitations of purely adversarial reasoning in RMDPs, we draw inspiration from advances in reactive stochastic games~\citep{10175747,10.1007/978-3-031-56940-1_17}. %
In this framework, a policy is deemed \emph{winning}, \emph{dominant}, or \emph{\besteffort} if it succeeds against \emph{all}, the \emph{maximum} subset, or a \emph{maximal} subset of the environment policies, respectively.
Yet, these papers consider games where only the graph of the model is known and the probabilities are unconstrained, as opposed to RMDPs, where the uncertainty is captured by bounded sets of distributions.

In this paper, we leverage the concepts of dominance and \besteffort to define a refined policy selection criterion for RMDPs, which we term \emph{\optimalrobust \besteffort} (\RBE).
An \RBE policy satisfies two properties:
\begin{enumerate*}[label=(\arabic*)]
\item it achieves an optimal expected return under the \emph{worst-case} transition probabilities; and
\item it is not {dominated} by any other policy, \ie is \besteffort.
\end{enumerate*}
Here, one policy is said to dominate another if it performs at least as well across the entire uncertainty set and strictly better in at least one instance of the transition probabilities from the uncertainty set.
This \besteffort perspective offers a principled tie-breaker among \optimalrobust policies, favoring those achieving a maximal expected return under non-adversarial transition probabilities.
Thus, \RBE policies preserve robust optimality---unlike approaches that update the uncertainty set---while also improving performance in non–fully adversarial environments.

\paragraph{Contributions.}
We introduce the class of {\optimalrobust \besteffort} (\RBE) policies for RMDPs.
These policies combine the worst-case guarantees of standard robust policies with the refinement offered by \besteffort reasoning, ensuring strong performance even when the environment is not fully adversarial.
Specifically, our key contributions are as follows:
\begin{itemize}
    \item We formalize the notions of dominant and \besteffort policies within the context of RMDPs (\cref{sec:definition}).
    \item We present a full characterization of \RBE policies and an efficient algorithm to compute them with small overhead to standard robust value iteration (\cref{sec:value_functions,sec:characterization}).
    \item We empirically demonstrate the feasibility of our techniques as a tie-breaker in robust value iteration (\cref{sec:implementation}).
\end{itemize}
We postpone a detailed discussion of related work to \cref{sec:related}.

\section{Preliminaries}\label{sec:preliminaries}

We write $\dotp{u}{v} \coloneqq \sum_{x \in X} u(x) v(x)$ for the dot product between the functions $u,v \colon X \to \RR$.
The cardinality of a set $X$ is written as $|X|$.
A probability distribution over a set $X$ is a function $\mu \colon X \to [0,1]$ such that $\sum_{x \in X} \mu(x) = 1$.
The set of all probability distributions over $X$ is denoted by $\dist{X}$.

\subsection{Markov Decision Processes}
\label{sec:MDPs}
We consider Markov decision processes (MDPs) with discounted rewards, defined as follows~\citep{DBLP:books/wi/Puterman94}.

\begin{definition}[MDP]\label{def:MDP}
    An MDP is a tuple $\MDP$, where $S$ is a finite set of states, $\initState \in \dist{S}$ is the initial distribution, $A$ is a finite set of actions, $P \colon S \times A \to \dist{S}$ is a transition function, $R \colon S \times A \to \RRnn$ is a state-action reward function, and $\gamma \in (0,1)$ is a discount factor.
\end{definition}

The actions in an MDP are chosen by a (randomized) policy $\policy \colon S \to \dist{A}$. %
We write $\policySpace$ for the set of all policies and simplify $\policy(s)(a)$ as $\policy(s,a)$.
The objective in an MDP is to compute a policy $\policy$ that maximizes the expected return $\rho^\policy_P$: %
\begin{equation}
    \rho^\policy_P 
    \coloneqq \sum\nolimits_{s \in S} \initState(s) V^\policy_P(s)
    = \langle \initState, V^\policy_P \rangle,
\end{equation}
where the value function $V^\policy_P \colon S \to \RR$ is defined as
\begin{equation*}
    V^\policy_P(s) \coloneqq \bE \left[ \sum\nolimits_{t=0}^\infty \gamma^t R^\policy(s_t) \,\Big|\, s_0 = s, s_{t+1} \sim P^\policy(s_t) \right],
\end{equation*}
with the transition and rewards functions for $\policy$ given as
\begin{align}
    \label{eq:P_policy}
    P^\policy(s) &\coloneqq \sum\nolimits_{a \in A} \policy(s,a) P(s,a) \in \dist{S}, 
    \\
    \label{eq:R_policy}
    R^\policy(s) &\coloneqq \sum\nolimits_{a \in A} \policy(s,a) R(s,a) \in \RR_{\geq 0}.
\end{align}
This value function is the fixed point of the Bellman operator $\cT^\policy_P$~\cite{DBLP:books/wi/Puterman94}, which is defined for all states $s \in S$ as
\[
(\cT^\policy_P V)(s) \coloneqq \left[ R^\policy(s) + \dotp{\gamma P^\policy(s)}{V} \right],
\]
whereas the optimal value $V^\star_P \coloneqq \max_{\policy \in \policySpace} V^\policy_P$ is the fixed point of the optimal Bellman operator $\cT^\star_P$ defined as
\[
(\cT^\star_P V)(s) \coloneqq \max_{\policy \in \policySpace} \cT^\policy_P V(s).
\]
Thus, the sequences $V^\policy_{n+1} \coloneqq \cT^\policy_P V^\policy_n$ and $V^\star_{n+1} \coloneqq \cT^\star_P V^\star_n$ converge to their respective fixed points, \ie $\lim_{n \to \infty} V^\policy_n = V^\policy_P$ and $\lim_{n \to \infty} V^\star_n = V^\star_P$.
Subsequently, an optimal policy can be computed as $\policy^\star_P \in \argmax_{\policy \in \policySpace} \cT^\policy_P V^\star_P$.

\subsection{Robust Markov Decision Processes}
\label{sec:RMDPs}
\emph{Robust MDPs} (RMDPs) extend MDPs with \emph{sets of transition probabilities}~\citep{DBLP:journals/mor/Iyengar05,DBLP:journals/ior/NilimG05}.
In an RMDP, the transition function is chosen from a set $\cP \subseteq \{ P \colon S \times A \to \dist{S} \}$ of transition functions, called the \emph{uncertainty set} (also known as the  \emph{ambiguity set}).

\begin{definition}[RMDP]\label{def:RMDP}
    A robust MDP (RMDP) is a tuple $\RMDP$, where $S$, $\initState$, $A$, $R$, and $\gamma$ are defined as in an MDP, and where $\cP \subseteq \{ P \colon S \times A \to \dist{S} \}$ is a set of transition functions, called the \emph{uncertainty set}.
\end{definition}

The robust expected return $\rho^\policy_\cP$ for the policy $\policy$ is defined as the worst-case expected return over the uncertainty set $\cP$:
\begin{equation}
    \label{eq:RMDP:robust_return}
    \rho^\policy_\cP \coloneqq \min_{P \in \cP} \rho^\policy_P.
\end{equation}
The standard objective in an RMDP is to find an \emph{\optimalrobust policy} $\policy^\star_\cP$ maximizing the robust expected return~$\rho^\star_\cP$: 
\begin{equation}
    \label{eq:RMDP:optimal_policy}
    \policy^\star_\cP \in \argmax_{\policy \in \policySpace} \rho^\policy_\cP, \quad
    \rho^\star_\cP \coloneqq \max_{\policy \in \policySpace} \rho^\policy_\cP.
\end{equation}
Unfortunately, solving \cref{eq:RMDP:robust_return,eq:RMDP:optimal_policy} is NP-hard for general uncertainty sets $\cP$, even if they are convex~\citep{DBLP:journals/mor/WiesemannKR13}.
Thus, $\cP$ is commonly assumed to be decomposable over states and/or state-action pairs, which is also known as \emph{rectangularity} of the uncertainty set.

\begin{definition}[Rectangularity]
    \label{def:rectangularity}
    The uncertainty set $\cP$ is \emph{\srectangular} if it can be decomposed state-wise as $\cP = \bigtimes_{s \in S} \cP_s$, where $\cP_s \subseteq \{ P \colon \Actions \to \dist{S} \}$.
    Moreover, $\cP$ is \emph{$(s,a)$-rectangular} if it can be decomposed state-action-wise as $\cP = \bigtimes_{s \in S, a \in A} \cP_{s,a}$, where $\cP_{s,a} \subseteq \dist{S}$.
\end{definition}

$(s,a)$-rectangularity is a special case of \srectangularity.
\begin{assumption}
Throughout the paper, the uncertainty set $\cP$ of an RMDP is assumed to be \srectangular.
\end{assumption}

Under \srectangularity, optimal policies may need to be randomized~\citep[Prop.~1]{DBLP:journals/mor/WiesemannKR13}.
Our definitions follow the usual semantics that the environment knows the stochastic policy of the agent but not the actual actions sampled from this policy, known as the \emph{environment first} (or \emph{nature first}) semantics~\citep{Suilen2024RMDPs}.

\paragraph{Robust value iteration.}
Under \srectangularity, for every policy $\policy$, there is a robust value function $V^\policy_\cP \colon S \to \RR$ that satisfies $V^\policy_\cP(s) \coloneqq \min_{P \in \cP} V^\policy_P(s)$ for all $s \in S$~\citep{DBLP:journals/mor/WiesemannKR13}. 
This value function $V^\policy_\cP$ is the fixed point of the robust Bellman operator $\cT^\policy_\cP$ for every $s \in S$:
\[
(\cT^\policy_\cP V)(s) \coloneqq \min_{P \in \cP_S} \left[ R^\policy(s) + \dotp{\gamma P^\policy(s)}{V} \right].
\]
Similarly, there exists an \optimalrobust value function $V^\star_\cP \coloneqq \max_{\policy \in \policySpace} V^\policy_\cP$, which is the fixed point of the \optimalrobust Bellman operator $\cT^\star_\cP$, defined for all $s \in S$ as
\[
(\cT^\star_\cP V)(s) \coloneqq \max_{\policy \in \policySpace} \cT^\policy_\cP V(s).
\]
\emph{Robust value iteration} leverages these fixed points so that the sequences $V^\policy_{n+1} \coloneqq \cT^\policy_\cP V^\policy_n$ and $V^\star_{n+1} \coloneqq \cT^\star_\cP V^\star_n$ converge to their respective fixed points, \ie $\lim_{n \to \infty} V^\policy_n = V^\policy_\cP$ and $\lim_{n \to \infty} V^\star_n = V^\star_\cP$.
Subsequently, an \optimalrobust policy can be computed as $\policy^\star_\cP \in \argmax_\policy \cT^\policy_\cP V^\star_\cP$.

\section{\BestEffort Policies in RMDPs}\label{sec:definition}

The \optimalrobust policy in \cref{eq:RMDP:optimal_policy} assumes the choice of transition function from the uncertainty set to be fully adversarial.
Here, we introduce \emph{dominance} and \emph{\besteffort} as the basis for a policy selection criterion that also considers non-adversarial scenarios.
These notions have been used in uncertain stochastic games~\citep{10175747}, but, as we discuss in~\cref{sec:related}, these results do not carry over to RMDPs.

\subsection{Dominant and \BestEffort Policies}
In this section, we tailor the definitions of \emph{dominant} and \emph{\besteffort} policies from~\citet{10175747} to RMDPs.
The first concept is that of \emph{dominance} between policies.
\begin{definition}[Dominance]
    \label{def:rmdp:dominance}
    Let $\policy, \policy' \in \policySpace$ be policies for the RMDP $\rmdp$.
    The policy $\policy$ \emph{dominates} $\policy'$, written $\policy \geq_{\cP} \policy'$, if and only if $\rho^\policy_P \geq \rho^{\policy'}_P$ for all $P \in \cP$.
\end{definition}
Intuitively, $\policy$ dominates $\policy'$ if $\policy$ does not perform worse than $\policy'$ under any transition function $P \in \cP$.
If, in addition, the policy $\policy$ also attains a \emph{strictly higher} expected return in some $P \in \cP$, then $\policy$ \emph{strictly} dominates~$\policy'$:
\begin{definition}[Strict dominance]
    \label{def:rmdp:strict_dominance}
    Let $\policy, \policy' \in \policySpace$ be policies for RMDP $\rmdp$.
    Policy $\policy$ dominates $\policy'$, written $\policy >_{\cP} \policy'$, if and only if $\policy \geq_{\cP} \policy'$ and there exists $P' \in \cP$ s.t. $\rho^\policy_{P'} > \rho^{\policy'}_{P'}$.
\end{definition}
We say that the policy $\policy$ is (strictly) dominant in the RMDP $\rmdp$ if it (strictly) dominates every other policy $\policy' \in \policySpace \setminus \{\policy\}$.
Next, we say that a policy is \emph{\besteffort} if there is no other policy that dominates it.
\begin{definition}[\Besteffort]
    \label{def:rmdp:besteffort}
    A policy $\policy \in \policySpace$ for the RMDP $\rmdp$ is \emph{\besteffort} if there is no $\policy' \in \policySpace$ such that $\policy' >_{\cP} \policy$.
    We denote by $\policySpace_\BE \subseteq \policySpace$ the set of all \besteffort policies.
\end{definition}
A policy is \besteffort if there is no other policy that is strictly better for some $P \in \cP$ and not worse for all $P \in \cP$.
In other words, a \besteffort policy cannot be improved without also decreasing the expected return under some transition function.
\Besteffort policies are \emph{incomparable} with respect to the dominance order, \ie for all $\policy, \policy' \in \policySpace_\BE,\ \policy \neq \policy'$, we have both $\policy \not\geq_\cP \policy'$ and $\policy \not\leq_\cP \policy'$.

\pgfmathdeclarefunction{funA}{3}{%
\pgfmathparse{#3 * #1 * #2 + #3 * (1-#1) * (2*#2)}
}
\pgfmathdeclarefunction{funB}{3}{%
\pgfmathparse{#3 * #1 * (1-#2) + #3 * (1-#1) * (1-2*#2)}
}
\pgfmathdeclarefunction{funV}{3}{%
\pgfmathparse{(funA(#1,#2,#3)/(1-0.5*#3)) / (1-(0.5*#3)/(1-0.5*#3) * funA(#1,#2,#3) - funB(#1,#2,#3))}
}

\begin{figure}[t!]
	\centering{
            \scalebox{0.84}{%
            \begin{tikzpicture}[
    state/.append style={inner sep=0pt, inner sep=0pt, minimum size=20pt}, 
    >=stealth,
    bobbel/.style={minimum size=1mm,inner sep=0pt,fill=black,circle},
    mynode/.style={rectangle,fill=white,anchor=center}]

    \draw [draw=none, use as bounding box] (1mm,-20mm) rectangle (42mm,13.5mm);

    \node[state] (s1) at (1,0) {$s_1$};
    \node[state] (s2) at ($(s1) + (2.5,0)$) {$s_2$};
    \draw[<-] (s1.west) -- node[above, pos=0.6]{$\initState$} +(-0.5,0);
    
    \node[bobbel] (s1a1) at ($(s1) + (1,0.4)$) {};
    \draw (s1) -- node[above]{$a_1$} (s1a1);
    \draw (s1a1) edge[->] node[above]{$\xi$} (s2);
    \draw (s1a1) edge[->, bend right = 80, looseness=1.5] node[above]{$1-\xi$} (s1.north);
    
    \node[bobbel] (s1a2) at ($(s1) + (1,-0.4)$) {};
    \draw (s1) -- node[below]{$a_2$} (s1a2);
    \draw (s1a2) edge[->] node[below]{$2\xi$} (s2);
    \draw (s1a2) edge[->, bend left = 80, looseness=1.5] node[below]{$1-2\xi$} (s1.south);

    \node[bobbel] (s2a1) at ($(s2) + (-0.6,-1.5)$) {};
    \draw (s2) -- node[left,pos=0.7]{$a$} (s2a1);
    \draw (s2a1) edge[->, out=190, in=-125, looseness=1.5] node[below, pos=0.15]{$0.5$} (s1.south west);
    \draw (s2a1) edge[->, bend right = 40, looseness=1] node[below, pos=0.3, yshift=-0.1cm]{$0.5$} (s2);

    \node[state, draw=none] (invisible) at (1,-1.8) {};
\end{tikzpicture}%
%
            \begin{tikzpicture}
    \draw [draw=none, use as bounding box] (-8mm,-4.3mm) rectangle (36.5mm,30.5mm);
        
    \begin{axis}[
        view={30}{30},
        width=5cm,
        height=4.6cm,
        xlabel=$\beta$, 
        ylabel=$\xi$, 
        zlabel=$\rho^\beta_\xi$, 
        x label style={rotate=0, yshift=0.4cm, xshift=-0.3cm},
        y label style={rotate=0, yshift=0.3cm, xshift=-0.1cm},
        z label style={rotate=-90, xshift=0.2cm, yshift=0cm},
        ztick={0,3,6},
        zmin=0,
        ]
        \addplot3[
          surf,
          domain=0:1,
          domain y=0:0.5,
        ]
        {funV(x,y,0.9)};

        \addplot3 [
            domain=0:0.5, 
            samples y=1, 
            darkgreen,
            smooth,
            ultra thick,
            dashdotted,
        ] 
        ({1}, {x}, {funV(1,x,0.9)});

        \addplot3 [
            domain=0:0.5, 
            samples y=1, 
            purple!80!black,
            smooth,
            ultra thick,
            loosely dashed,
        ] 
        ({0}, {x}, {funV(0,x,0.9)});

        \addplot3[
            only marks,
            mark=*,
            mark size=1.5pt,
            color=black,
        ] coordinates {
            (1,0,0)
        };
        \addplot3[
            only marks,
            mark=square*,
            mark size=1.5pt,
            color=black,
        ] coordinates {
            (1,0.1667,0)
        };
        \addplot3[
            only marks,
            mark=diamond*,
            mark size=2pt,
            color=black,
        ] coordinates {
            (1,0.3333,0)
        };
        \addplot3[
            only marks,
            mark=triangle*,
            mark size=2pt,
            color=black,
        ] coordinates {
            (1,0.5,0)
        };
        
      \end{axis}
\end{tikzpicture}
            }
        }
	\caption{\emph{Left:} An RMDP with two states, where the policy is fully defined by the probability $\beta \coloneqq \policy(s_1,a_1)$ of choosing $a_1$ in $s_1$.
    The reward function is defined as $R(s_1,a_1) = R(s_1,a_2) = 0$ and $R(s_2, a) = 1$.
    \emph{Right:} The expected return $\rho^\beta_\xi$ as a function of $\beta$ and $\xi \in [0,0.5]$.
    All policies are \optimalrobust, but only the policy with $\beta = 0$ is \besteffort.}
	\label{fig:rmdp1}
\end{figure}
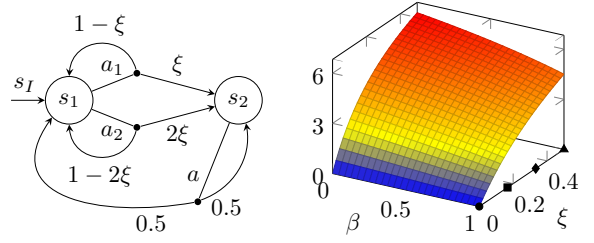

\subsection{\OptimalRobust Does Not Imply \BestEffort}
\label{sec:notallBE}

In general, \optimalrobust policies for RMDPs are not \besteffort, as shown by the two-state RMDP in \cref{fig:rmdp1} with reward function $R(s_1,a_1) = R(s_1,a_2) = 0$, $R(s_2, a) = 1$.
In this RMDP, the only action choice is between $a_1$ and $a_2$ in state $s_1$, so the stochastic policy $\policy$ is completely defined by the probability $\beta \coloneqq \policy(s_1,a_1) \in [0,1]$ of choosing $a_1$ in $s_1$.
Similarly, the (\srectangular) uncertainty set $\cP$ is fully defined by the parameter $\xi \in \Xi = [0,0.5]$.
As a result, we may simplify the notations from the preceding definitions by replacing $\pi$ with $\beta$, $P$ with $\xi$, and $\cP$ with $\Xi$.
The value function $V^\beta_\xi \colon \{s_1, s_2\} \to \RR$ depends on $\beta$ and $\xi$ and is defined as
\begin{align*}
    V^\beta_\xi(s_1) &= \gamma \big( \beta \big[\xi V^\beta_\xi(s_2) + (1-\xi) V^\beta_\xi(s_1) \big]
    \\
    &\qquad + (1-\beta) \big[ 2 \xi V^\beta_\xi(s_2) + (1-2 \xi) V^\beta_\xi(s_1) \big] \big),
    \\
    V^\beta_\xi(s_2) &= 1 + \gamma \big[ 0.5 V^\beta_\xi(s_1) + 0.5 V^\beta_\xi(s_2) \big].
\end{align*}
Solving for the value $\rho^\beta_\xi = V^\beta_\xi(s_1)$ with $\gamma = 0.9$ yields the surface in the right half of \cref{fig:rmdp1}.
This surface shows the expected return $\rho^\beta_\xi$ for all $\beta \in [0,1]$ and $\xi \in [0,0.5]$.
The worst-case expected return is zero and is attained at $\xi=0$, regardless of the value of $\beta$.
Thus, all policies in this RMDP are \optimalrobust.
Nevertheless, only the policy for $\beta = 0$ is \besteffort, because for all $\xi > 0$, the expected return for $\beta = 0$ is \emph{strictly higher} than for any $\beta > 0$.
In other words, the policy defined as $\policy(s_1,a_1) = 0$, $\policy(s_1,a_2) = 1$ \emph{strictly dominates} all other policies $\policy' \neq \policy$, that is, $\policy >_\Xi \policy'$.

\subsection{Problem Statement}
Above, we have shown that not all \optimalrobust policies are also \besteffort.
This observation motivates the next core problem, which we shall solve in the remainder of this paper.

\begin{problem} %
    \label{prob:statement}
    For a given RMDP $\rmdp$, compute a policy $\policy^\star_\BE$ that is \optimalrobust and \besteffort, \ie
    \[
    \policy^\star_\BE \in \argmax_{\policy \in \policySpace} \rho^\policy_\cP \,\, \text{such that} \,\, \nexists \policy' \in \policySpace \setminus \{\policy^\star_\BE\}, \, \policy' >_\cP \policy^\star_\BE.
    \]
\end{problem}

We call a policy that satisfies \cref{prob:statement} \emph{\optimalrobust \besteffort} (\RBE).
In RMDPs with multiple \optimalrobust policies (as in \cref{fig:rmdp1}), the \besteffort criterion offers a principled tiebreaker, favoring a policy that attains a maximal performance under non-adversarial transition probabilities.

\begin{remark}%
    \label{fn:proofs}For clarity and due to space constraints, all proofs are provided in \cref{appendix:proofs}.
\end{remark}

\pgfmathdeclarefunction{funVnoXi}{3}{%
\pgfmathparse{((#3 * #2) / (1 - #3/2)) / (1 - (#3^2/2 * #2) / (1-#3/2) - #3 * #1)}
}

\pgfmathdeclarefunction{funVdelA}{3}{%
\pgfmathparse{((#3^2 * #2) / (1-#3/2)) / (1 - (#3^2/2 * #2) / (1-#3/2) - #3 * #1)^2}
}

\pgfmathdeclarefunction{funVdelB}{3}{%
\pgfmathparse{( (#3/(1-#3/2) * (1 - (#3^2/2 * #2) / (1-#3/2) - #3 * #1) - #3 / (1-#3/2) * #2 * -(#3^2/2)/(1-#3/2) ) / (1 - (#3^2/2 * #2) / (1-#3/2) - #3 * #1)^2}
}

\pgfmathdeclarefunction{funVdelDir}{5}{%
\pgfmathparse{#4 * funVdelA(#1,#2,#3) + #5 * funVdelB(#1,#2,#3)}
}

\def\tikzSamples{30}
\def\tikzOpacity{0.6}
\def\tikzArrowColor{black} %

\pgfplotsset{
    tikzstyle/.style={
        view={-35}{25},
        width=5cm,
        height=5cm,
        x label style={rotate=15, yshift=0.2cm, xshift=-0.25cm},
        y label style={rotate=-24, yshift=0.38cm, xshift=0.5cm},
        z label style={rotate=-90, anchor=south, xshift=15pt, yshift=1cm},
        xtick={0,0.5,1},
        ytick={0,0.5,1},
        ztick={0,3,6},
        zmin=0,
        zmax=7.5,
        restrict z to domain=0:1000,
        point meta min=0,
        point meta max=7,
    }
}

\section{Representation of Robust Value Functions}
\label{sec:value_functions}

We first introduce a change in perspective to the value function, which we will use in \cref{sec:characterization} to solve \cref{prob:statement}.
Instead of using shared variables to represent dependencies between probabilities (such as $\xi$ in \cref{fig:rmdp1}), we label each transition with its own probability $p(s,a)(s')$ and encode dependencies in the uncertainty set $\cP$.
For instance, we can equally represent the RMDP in \cref{fig:rmdp1} using the uncertainty set
\begin{align*}
{}&{}\cP_{s_1} = \{ (P_{s_1} \colon A \to \dist{S}) {}:{} 
	p(a_1)(s_1) + p(a_1)(s_2) = 1, 
	\\ 
    & \,\,\,\, p(a_2)(s_1) + p(a_2)(s_2) = 1, \,\, 
      p(a_1)(s_2) = 0.5 p(a_2)(s_2) \}.
\end{align*}
We aim to reason about the expected return when the transition function is fixed \emph{in all but one state}.
To this end, we introduce the notion of a partial transition function.

\begin{definition}[Partial transition function]
    \label{def:partial_P}
    Let $\cP$ be an \srectangular uncertainty set and let $\bar{s} \in S$ be a state.
    A partial transition function $P_{-\bar{s}}$ for state $\bar{s}$ is defined as
    $
    P_{-\bar{s}} = \bigtimes_{s \in S \setminus \{ \bar{s} \}} P_s,
    $
    where $P_s \in \cP$ for all $s \in S \setminus \{ \bar{s} \}$.
\end{definition}

A partial transition function has the form $P_{-\bar{s}} \colon (\States \setminus \{\bar{s}\}) \times \Actions \to \dist{S}$. %
Thus, to complete $P_{-\bar{s}}$ with $P_{\bar{s}} \in \cP_{\bar{s}}$ for the missing state $\bar{s}$, we take the product $P_{-\bar{s}} \times P_{\bar{s}}$.
Similarly, we write $P_{-\bar{s}} \times \cP_{\bar{s}}$ for the set of all completions, such that $P_{-\bar{s}} \times P_{\bar{s}} \in P_{-\bar{s}} \times \cP_{\bar{s}}$.
Using this notation, we define the following value function in a fixed state $\bar{s}$, when the transition probabilities are fixed in all states but $\bar{s}$.

\begin{definition}[Parametric value function]
    \label{def:parametric_V}
    The value in state $\bar{s}$ is a function of the completion $P_{\bar{s}} \in \cP_{\bar{s}}$ of the partial transition function $P_{-\bar{s}}$ and is defined as $Z^\policy_{P,\bar{s}}(P_{\bar{s}}) = V^\policy_{P_{-\bar{s}} \times P_{\bar{s}}}(\bar{s})$.
\end{definition}

\begin{example}
Consider again the RMDP from \cref{fig:rmdp1} with the policies given by $\beta = 0$ and $\beta = 1$.
The value functions $Z^\policy_{P,s_1}$ for these two policies are, respectively, shown in the left and right halves of \cref{fig:rmdp1:value_function2}.
For $\beta=1$, the value depends only of the transition probabilities related to action $a_1$ (and for $\beta=0$ only of those related to $a_2$).
In both plots, the dashed line in the bottom plane shows the set of valid distributions in $\cP_{s_1}$, where the marked points coincide with those on the $\xi$-axis in \cref{fig:rmdp1}.
The green (left) and purple (right) curved lines show the expected return for the policies with $\beta=1$ and $\beta=0$, respectively, as a function of $\xi$ and coincide with the lines of the same color in \cref{fig:rmdp1}.
As in \cref{fig:rmdp1}, we observe that, for any $\xi > 0$, the policy for $\beta=1$ strictly dominates all other policies and is, thus, \besteffort.
\end{example}

\begin{figure}[t!]
	\centering
            \scalebox{0.85}{\begin{tikzpicture}
    \draw [draw=none, use as bounding box] (-4mm,-9.2mm) rectangle (37mm,34.8mm);
        
    \begin{axis}[tikzstyle,
        xlabel={$P(s_1,a_1)(s_1)$}, 
        ylabel={$P(s_1,a_1)(s_2)$}, 
        zlabel={$Z^\pi_{P,s_1}$}, 
        ]

        \addplot3 [
            domain=0.5:1, 
            samples y=1, 
            black!90,
            smooth,
            thick,
            dashed,
        ] 
        ({x}, {1-x}, {0});
        \addplot3[
            thick,
            black!90,
            smooth,
            dashed,
        ] coordinates {
            ({0.5},{0.5},{0})
            ({0.5},{0.5},{funVnoXi(0.5,0.5,0.9)})
        };
        \addplot3[
            thick,
            black!90,
            smooth,
            dashed,
        ] coordinates {
            ({0.6667},{0.3333},{0})
            ({0.6667},{0.3333},{funVnoXi(0.6667,0.3333,0.9)})
        };
        \addplot3[
            thick,
            black!90,
            smooth,
            dashed,
        ] coordinates {
            ({0.8333},{0.1667},{0})
            ({0.8333},{0.1667},{funVnoXi(0.8333,0.1667,0.9)})
        };
        
        \addplot3[
            only marks,
            mark=*,
            mark size=1.5pt,
            color=black!90,
            opacity=0.5
        ] coordinates {
            ({0.5},{0.5},{0})
        };
        \addplot3[
            only marks,
            mark=square*,
            mark size=1.5pt,
            color=black!90,
            opacity=0.5
        ] coordinates {
            ({0.6667},{0.3333},{0})
        };
        \addplot3[
            only marks,
            mark=diamond*,
            mark size=2.5pt,
            color=black!90,
            opacity=0.5
        ] coordinates {
            ({0.8333},{0.1667},{0})
        };

        \addplot3[
          surf,
          domain=0:1,
          domain y=0:1,
          samples=\tikzSamples,
          opacity=\tikzOpacity,
        ]
        {funVnoXi(x,y,0.9)};

        \addplot3[
            thick,
            black!90,
            smooth,
            densely dotted,
        ] coordinates {
            ({1},{0},{funVnoXi(0.5,0.5,0.9)})
            ({0},{1},{funVnoXi(0.5,0.5,0.9)})
        };

        \addplot3 [
            domain=0.5:1, 
            samples y=1, 
            darkgreen,
            smooth,
            ultra thick,
            dashdotted,
        ] 
        ({x}, {1-x}, {funVnoXi(x,(1-x),0.9)});
        \addplot3 [
            domain=0:0.5, 
            samples y=1, 
            black!90,
            smooth,
            densely dotted,
        ] 
        ({x}, {1-x}, {funVnoXi(x,(1-x),0.9)});

    \addplot3[
        only marks,
        mark=*,
        mark size=2.5pt,
        color=darkgreen,
    ] coordinates {
        ({0.5},{0.5},{funVnoXi(0.5,0.5,0.9)})
    };
    \addplot3[
        only marks,
        mark=square*,
        mark size=2.5pt,
        color=darkgreen,
    ] coordinates {
        ({0.6667},{0.3333},{funVnoXi(0.6667,0.3333,0.9)})
    };
    \addplot3[
        only marks,
        mark=diamond*,
        mark size=3pt,
        color=darkgreen,
    ] coordinates {
        ({0.8333},{0.1667},{funVnoXi(0.8333,0.1667,0.9)})
    };
    \addplot3[
        only marks,
        mark=triangle*,
        mark size=3pt,
        color=darkgreen,
    ] coordinates {
        ({1},{0},{funVnoXi(1,0,0.9)})
    };

    \node[anchor=west] at (axis cs:0.5,0.5,{funVnoXi(0.5,0.5,0.9)}) [xshift=-0.2cm, yshift=0.3cm] {$\xi=0.5$};

  \end{axis}
\end{tikzpicture}%
\hspace{0.3cm}
\begin{tikzpicture}
    \draw [draw=none, use as bounding box] (-4mm,-9.2mm) rectangle (37mm,34.8mm);
        
    \begin{axis}[tikzstyle,
        xlabel={$P(s_1,a_2)(s_1)$}, 
        ylabel={$P(s_1,a_2)(s_2)$}, 
        zlabel={$Z^{\pi'}_{P,s_1}$},
        ]

        \addplot3 [
            domain=0:1, 
            samples y=1, 
            black!90,
            smooth,
            thick,
            dashed,
        ] 
        ({x}, {1-x}, {0});
        \addplot3[
            thick,
            black!90,
            smooth,
            dashed,
        ] coordinates {
            ({0},{1},{0})
            ({0},{1},{funVnoXi(0,1,0.9)})
        };
        \addplot3[
            thick,
            black!90,
            smooth,
            dashed,
        ] coordinates {
            ({0.3333},{0.6667},{0})
            ({0.3333},{0.6667},{funVnoXi(0.3333,0.6667,0.9)})
        };
        \addplot3[
            thick,
            black!90,
            smooth,
            dashed,
        ] coordinates {
            ({0.6667},{0.3333},{0})
            ({0.6667},{0.3333},{funVnoXi(0.6667,0.3333,0.9)})
        };
        
        \addplot3[
            only marks,
            mark=*,
            mark size=1.5pt,
            color=black!90,
            opacity=0.5
        ] coordinates {
            ({0},{1},{0})
        };
        \addplot3[
            only marks,
            mark=square*,
            mark size=1.5pt,
            color=black!90,
            opacity=0.5
        ] coordinates {
            ({0.3333},{0.6667},{0})
        };
        \addplot3[
            only marks,
            mark=diamond*,
            mark size=2.5pt,
            color=black!90,
            opacity=0.5
        ] coordinates {
            ({0.6667},{0.3333},{0})
        };

        \addplot3[
          surf,
          domain=0:1,
          domain y=0:1,
          samples=\tikzSamples,
          opacity=\tikzOpacity,
        ]
        {funVnoXi(x,y,0.9)};

        \addplot3[
            thick,
            black!90,
            smooth,
            densely dotted,
        ] coordinates {
            ({1},{0},{funVnoXi(0,1,0.9)})
            ({0},{1},{funVnoXi(0,1,0.9)})
        };

        \addplot3 [
            domain=0:1, 
            samples y=1, 
            purple!80!black,
            smooth,
            ultra thick,
            loosely dashed,
        ] 
        ({x}, {1-x}, {funVnoXi(x,(1-x),0.9)});

    \addplot3[
        only marks,
        mark=*,
        mark size=2.5pt,
        color=purple!80!black,
    ] coordinates {
        ({0},{1},{funVnoXi(0,1,0.9)})
    };
    \addplot3[
        only marks,
        mark=square*,
        mark size=2.5pt,
        color=purple!80!black,
    ] coordinates {
        ({0.3333},{0.6667},{funVnoXi(0.3333,0.6667,0.9)})
    };
    \addplot3[
        only marks,
        mark=diamond*,
        mark size=3pt,
        color=purple!80!black,
    ] coordinates {
        ({0.6667},{0.3333},{funVnoXi(0.6667,0.3333,0.9)})
    };
    \addplot3[
        only marks,
        mark=triangle*,
        mark size=3pt,
        color=purple!80!black,
    ] coordinates {
        ({1},{0},{funVnoXi(1,0,0.9)})
    };

    \node[anchor=west] at (axis cs:0,1,{funVnoXi(0,1,0.9)}) [yshift=0.2cm] {$\xi=0.5$};

  \end{axis}
\end{tikzpicture}}
        \caption{The value function $Z^\policy_{P,s_1}$ in state $s_1$ for the RMDP from \cref{fig:rmdp1}, shown for the policies with $\beta=1$ (left) and $\beta=0$ (right).
        The curved lines show the expected return as the parameter $\xi$ in \cref{fig:rmdp1} ranges from $0$ to $0.5$ (the line markers correspond with those on the $\xi$-axis in \cref{fig:rmdp1}).
    }
	\label{fig:rmdp1:value_function2}
\end{figure}
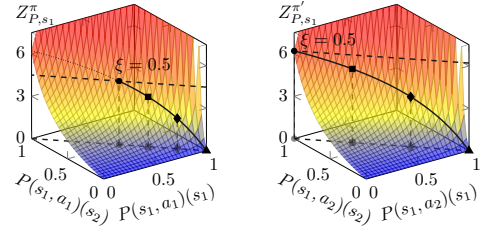

\section{Finding \OptimalRobust \BestEffort Policies}
\label{sec:characterization}
We now use the representation of the value function from \cref{def:parametric_V} to determine whether an \optimalrobust policy for a fixed RMDP $\rmdp = \RMDP$ is \besteffort.

\subsection{Existence of ORBE Policies}
\label{sec:char:existence}

We first establish in \cref{th:rmdp:existence:robustbesteffort} that, for any \srectangular RMDP, the set of \RBE policies is nonempty.
Intuitively, this result holds because the dominance relation imposes a partial order over policies, ensuring the existence of maximal (\ie \besteffort) ones that under an adversarial environment must also be \optimalrobust.
Furthermore, any optimal robust policy $\policy \in \policySpace^\star$ cannot be dominated by a policy that is not optimal robust.
Thus, an \RBE policy always exists.

\begin{restatable}[Existence of \RBE policies]{theorem}{orbeexistence}
\label{th:rmdp:existence:robustbesteffort}
For any RMDP, the intersection of the sets of \emph{\optimalrobust policies} $\policySpace^\star$ and \emph{\besteffort policies} $\policySpace_\BE$ is nonempty.
\end{restatable}

Note that the existence of \besteffort policies in RMDPs does not directly follow from the results for synthesis in stochastic environments in~\citet{10175747}; see~\Cref{sec:related}.

\subsection{Characterizing ORBE policies}
In \cref{thm:ORBE_char}, we provide a sufficient condition for \RBE policies, used as a foundation in the remainder of the section.

\begin{theorem}[\RBE policy]
    \label{thm:ORBE_char}
    Given an \optimalrobust policy $\policy^\star \in \policySpace^\star \coloneqq \argmax_{\policy \in \policySpace} \rho^\policy_\cP$, if there exists $P \in \cP$ such that $\rho^{\policy^\star}_P > \rho^{\policy'}_P$ for all $\policy' \in \policySpace^\star \setminus \{\policy^\star\}$, then $\policy^\star$ is \RBE.
\end{theorem}

\begin{proof}
    First, $\policy^\star$ is \optimalrobust by definition.
    Second, to show that $\policy^\star$ is also \besteffort, we must show there is no other policy $\policy' \in \policySpace \setminus \{\policy^\star\}$ that strictly dominates $\policy^\star$.
    By construction, $\policy^\star$ cannot be dominated by any $\policy' \in \policySpace^\star \setminus \{\policy^\star\}$.
    For any other policy $\policy'' \in \policySpace \setminus \policySpace^\star$, we have $\rho^{\policy^\star}_\cP > \rho^{\policy''}_\cP$ and, moreover, as $\rho^{\policy^\star}_\cP = \min_{P \in \cP} \rho^\policy_P$ (cf.~\cref{eq:RMDP:robust_return}), it holds that $\rho^{\policy^\star}_{P'} \geq \rho^{\policy^\star}_\cP$ for all $P' \in \cP$.
    By letting $P' \in \argmin_{P \in \cP} \rho^{\policy''}_P$, we thus obtain $\rho^{\policy^\star}_{P'} \geq \rho^{\policy^\star}_\cP > \rho^{\policy''}_\cP = \rho^{\policy''}_{P'}$, which proves that $\policy'' \not>_{\cP} \policy^
    \star$.
    Thus, the policy $\policy^\star$ is \RBE.
\end{proof}

In the remainder of this section, we use \cref{thm:ORBE_char} to derive conditions under which an \optimalrobust policy is also \besteffort (and thus \RBE).
First of all, if an \optimalrobust policy is unique, then this policy is also \besteffort.
\begin{restatable}%
{corollary}{maxminUnique}
    \label{corr:char:maxmin_unique}
    Let $\policySpace^\star = \argmax_{\policy \in \policySpace} \rho^\policy_\cP$ be the set of \optimalrobust policies.
    If $\policySpace^\star$ is a singleton, then $\policy^\star \in \policySpace^\star$ is \RBE.
\end{restatable}

\subsubsection{ORBE via optimistic RVI.}
The second observation is that, if an \optimalrobust policy is not unique but further optimizing via robust value iteration (RVI) for the \emph{optimistic} (\ie maximizing) transition function does yield a unique optimum, then the resulting policy is also \besteffort.
\begin{restatable}%
{corollary}{maxmaxUnique}
    \label{corr:char:maxmax_unique}
    Let $\check\policySpace^\star = \argmax_{\policy \in \policySpace} \rho^\policy_\cP$ and let $\hat\policySpace^\star = \argmax_{\policy \in \check\policySpace^\star} \max_{P \in \cP} \rho^\policy_P$ be the set of policies that (within $\check\policySpace^\star$) maximize the expected return under the maximizing $P \in \cP$.
    If $\hat\policySpace^\star$ is a singleton, then $\policy^\star \in \hat\policySpace^\star$ is \RBE.
\end{restatable}

\begin{example}
    Consider again the RMDP in \cref{fig:rmdp1}.
    Even though all policies are optimal robust, only the policy for $\beta = 0$ is optimal under the maximizing transition function (which is attained for $\xi = 0.5$).
    Thus, the policy for $\beta = 0$, \ie always choosing action $a_2$, is \RBE.
\end{example}

\subsubsection{ORBE via derivatives.}
Another way to determine if a policy is \besteffort is to reason about the derivative of the value function.
Let $\nabla_{\mathbf{v}} f(x) = \mathbf{v}^\top \cdot \frac{\partial f(x)}{\partial x}$ be the \emph{directional derivative} of the function $f : \RR^n \to \RR$ in the direction $\mathbf{v} \in \RR^n$.
Recall from \cref{def:parametric_V} that $Z^\policy_{P,\bar{s}}$ is the value function in state $\bar{s}$ when the transition function is fixed in all states except $\bar{s}$.
The next result states that, if an \optimalrobust policy $\policy^\star$ leads, for every state $\bar{s}$, to a \emph{strictly higher} derivative of $Z^\policy_{P,\bar{s}}(P_{\bar{s}})$ than all other \optimalrobust policies, then $\policy^\star$ is \besteffort.
This derivative can be taken in any direction such that the perturbed $P_{\bar{s}}$ is still within the uncertainty set $\cP_{\bar{s}}$.

\begin{restatable}{corollary}{derivative}
    \label{corr:char:derivative}
    Let $\bar\policy \in \policySpace^\star = \argmax_{\policy \in \policySpace} \rho^{\bar\policy}_\cP$ be an \optimalrobust policy with minimizer $P^\star \in \argmin_{P \in \cP} \rho^{\bar\policy}_P$.
    Define $\policySpace^\star_{(1)} = \argmax_{\policy \in \policySpace^\star} \rho^\policy_{P^\star}$ and pick a policy $\policy^\star \in \policySpace^\star_{(1)}$.
    The policy $\policy^\star$ is \RBE if, for all states $\bar{s}$, there exists a vector $\mathbf{v} \in \RR^{|S|}$ such that $\exists \epsilon > 0, \, P^\star_{\bar{s}} + \epsilon \mathbf{v} \in \cP_{\bar{s}}$ and
    \begin{equation}
        \nabla_{\mathbf{v}} Z^{\policy^\star}_{P^\star,\bar{s}}(P^\star_{\bar{s}}) > \nabla_{\mathbf{v}} Z^{\policy'}_{P^\star,\bar{s}}(P^\star_{\bar{s}}) \enskip \forall \policy' \in \policySpace^\star_{(1)} \setminus \{ \policy^\star \}.
    \end{equation}
\end{restatable}
Intuitively, the condition that there exists $\epsilon > 0$ such that $P^\star_{\bar{s}} + \epsilon \mathbf{v} \in \cP_{\bar{s}}$ encodes that the vector $\mathbf{v}$ at the minimizing transition function $P^\star$ points inside the uncertainty set~$\cP_{\bar{s}}$.

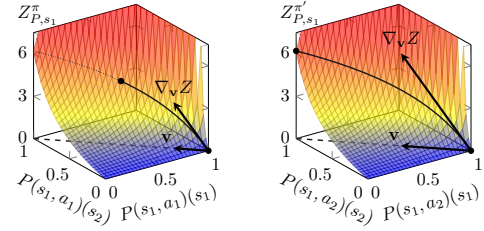
\begin{figure}[t!]
	\centering
            \scalebox{0.85}{\begin{tikzpicture}
    \draw [draw=none, use as bounding box] (-4mm,-9.2mm) rectangle (37mm,34.8mm);
        
    \begin{axis}[tikzstyle,
        xlabel={$P(s_1,a_1)(s_1)$}, 
        ylabel={$P(s_1,a_1)(s_2)$},
        zlabel={$Z^\pi_{P,s_1}$}, 
        ]

        \addplot3 [
            domain=0:1, 
            samples y=1, 
            black!90,
            smooth,
            thick,
            dashed,
        ] 
        ({x}, {1-x}, {0});

        \addplot3[
          surf,
          domain=0:1,
          domain y=0:1,
          samples=\tikzSamples,
          opacity=\tikzOpacity,
        ]
        {funVnoXi(x,y,0.9)};

        \addplot3 [
            domain=0.5:1, 
            samples y=1, 
            black!90,
            smooth,
            thick,
        ] 
        ({x}, {1-x}, {funVnoXi(x,(1-x),0.9)});
        \addplot3 [
            domain=0:0.5, 
            samples y=1, 
            black!90,
            smooth,
            densely dotted,
        ] 
        ({x}, {1-x}, {funVnoXi(x,(1-x),0.9)});

    \addplot3[
        only marks,
        mark=*,
        mark size=1.5pt,
        color=black,
    ] coordinates {
        ({0.5},{0.5},{funVnoXi(0.5,0.5,0.9)})
        ({1},{0},{funVnoXi(1,0,0.9)})
    };

    \addplot3[
        very thick,
        \tikzArrowColor,
        smooth,
        -stealth,
    ] coordinates {
        (1,0,0)
        (0.8,0.2,0)
    };
    \addplot3[
        very thick,
        \tikzArrowColor,
        smooth,
        -stealth,
    ] coordinates {
        (1,0,0)
        ({0.8},{0.2},{funVdelDir(1,0,0.9,-0.2,0.2)})
    };
    
    \node[anchor=east] at (axis cs:0.8,0.2,0) [yshift=0.2cm, xshift=0.1cm] {$\mathbf{v}$};
    \node[anchor=south] at (axis cs:0.8,0.2,{funVdelDir(1,0,0.9,-0.2,0.2)}) [yshift=0.0cm, xshift=0cm] {$\nabla_\mathbf{v} Z$};

  \end{axis}
\end{tikzpicture}%
\hspace{0.3cm}
\begin{tikzpicture}
    \draw [draw=none, use as bounding box] (-4mm,-9.2mm) rectangle (37mm,34.8mm);
        
    \begin{axis}[tikzstyle,
        xlabel={$P(s_1,a_2)(s_1)$}, 
        ylabel={$P(s_1,a_2)(s_2)$},
        zlabel={$Z^{\pi'}_{P,s_1}$}, 
        ]

        \addplot3 [
            domain=0:1, 
            samples y=1, 
            black!90,
            smooth,
            thick,
            dashed,
        ] 
        ({x}, {1-x}, {0});

        \addplot3[
          surf,
          domain=0:1,
          domain y=0:1,
          samples=\tikzSamples,
          opacity=\tikzOpacity,
        ]
        {funVnoXi(x,y,0.9)};

        \addplot3 [
            domain=0:1, 
            samples y=1, 
            black!90,
            smooth,
            thick,
        ] 
        ({x}, {1-x}, {funVnoXi(x,(1-x),0.9)});

    \addplot3[
        only marks,
        mark=*,
        mark size=1.5pt,
        color=black,
    ] coordinates {
        ({0},{1},{funVnoXi(0,1,0.9)})
        ({1},{0},{funVnoXi(1,0,0.9)})
    };

    \addplot3[
        very thick,
        \tikzArrowColor,
        smooth,
        -stealth,
    ] coordinates {
        (1,0,0)
        (0.6,0.4,0)
    };
    \addplot3[
        very thick,
        \tikzArrowColor,
        smooth,
        -stealth,
    ] coordinates {
        (1,0,0)
        ({0.6},{0.4},{funVdelDir(1,0,0.9,-0.4,0.4)})
    };

    \node[anchor=east] at (axis cs:0.6,0.4,0) [yshift=0.2cm, xshift=0.1cm] {$\mathbf{v}$};
    \node[anchor=south] at (axis cs:0.6,0.4,{funVdelDir(1,0,0.9,-0.4,0.4)}) [yshift=0.0cm, xshift=0cm] {$\nabla_\mathbf{v} Z$};

  \end{axis}
\end{tikzpicture}}
        \caption{The directional derivative $\nabla_\mathbf{v} Z^\policy_{P,s_1}$ for $\beta = 0$ (shown in the right half) is strictly larger than for any $\beta > 0$.
    Hence, we conclude that the policy for $\beta = 0$ is \RBE.}
	\label{fig:rmdp1:value_function3}
\end{figure}

\begin{example}
    Another way to characterize the \RBE policy of $\beta = 0$ in \cref{fig:rmdp1} is to compare the derivatives of $Z^\policy_{P,s_1}$ at the minimizing transition function $\xi = 0$.
    In this example, the only feasible direction is $\mathbf{v} = [-\alpha,\alpha]$,~$\alpha > 0$, as shown in \cref{fig:rmdp1:value_function3}.
    Any change in $P(s_1,a_1)(s_1)$ and $P(s_1,a_1)(s_2)$ causes a change twice as big in $P(s_1,a_2)(s_1)$ and $P(s_1,a_2)(s_1)$, visualized by the longer vectors in \cref{fig:rmdp1:value_function3}.
    Thus, the directional derivative for the policy with $\beta=0$ (\ie always choosing action $a_2$) is strictly larger than for all $\beta > 0$.
    Therefore, the policy with $\beta = 0$ is \RBE.
\end{example}

Conversely, we can consider the derivative of the value function under the policies $\policy^\star \in \hat{\policySpace}^\star$ (as defined in \cref{corr:char:maxmax_unique}) that are, besides being \optimalrobust, also optimal under the maximizing transition function.
In this case, if $\policy^\star$ leads to a strictly \emph{lower} directional derivative $\nabla_{\mathbf{v}} Z^{\policy^\star}_{P,\bar{s}}(P_{\bar{s}})$ than all other policies $\policy' \in \policySpace^\star \setminus \{\policy^\star\}$, then $\policy^\star$ is \besteffort.
This result and the proof are analogous to \cref{corr:char:derivative}, so we omit a formal statement due to space limitations.

\paragraph{Completeness.}
So far, we have shown that an optimal robust policy $\policy^\star$ is \besteffort if either of the following holds:
\begin{enumerate}
\item $\policy^\star$ is \emph{uniquely} optimal (in the minimizing or maximizing sense with respect to the transition function).
\item $\policy^\star$ yields a \emph{uniquely} highest (resp. lowest) directional derivative at the minimizing (resp. maximizing) $P^\star \in \cP$.
\end{enumerate}
In this section, we complete the characterization by showing that any policy that satisfies these conditions up to this uniqueness is also best-effort.
\Cref{thm:char:complete} formalizes this non-trivial result.
For conciseness, we defer the preliminaries needed for the proof of \Cref{thm:char:complete} to \cref{appendix:proof:completeness}.

\begin{restatable}[Computing ORBE policies] %
{theorem}{completeness}
    \label{thm:char:complete}
    Let $\policySpace^\star = \argmax_{\policy \in \policySpace} \rho^\policy_\cP$ be the \optimalrobust policies.
    Pick two transition functions $P^{(1)}, P^{(2)} \in \cP$ such that $P^{(1)} \neq P^{(2)}$ and, for all $\bar{s} \in S$, the line $g_{\bar{s}}(\lambda) = \lambda P^{(1)}_{\bar{s}} + (1-\lambda) P^{(2)}_{\bar{s}}$ intersects the relative interior\footnote{The relative interior of a convex set $X$ is defined as $\relinterior{X} \coloneqq \{ x \in X : \forall y \in X, \, \exists \lambda > 1. \,\, \lambda x + (1-\lambda)y \in X \}$.} of $\cP_{\bar{s}}$, %
    or $\cP_{\bar{s}} \cap \{g_{\bar{s}}(\lambda)\} = \cP_{\bar{s}}$.
    Define
    \begin{align*}
        \policySpace^\star_{(1)} = \argmax_{\policy \in \policySpace^\star} \rho^\policy_{P^{(1)}},
        \qquad
        \policySpace^\star_{(2)} = \argmax_{\policy \in \policySpace^\star_{(1)}} \rho^\policy_{P^{(2)}}.
    \end{align*}
    Choose a policy $\policy^\star \in \policySpace^\star_{(2)}$ s.t., for all $\bar{s} \in \States$, it holds that
    \begin{subequations}
    \label{eq:completeness:deriv}
    \begin{align}
        \label{eq:completeness:deriv-max}
        \nabla_{\mathbf{v}} Z^{\policy^\star}_{P^{(1)},\bar{s}}(P^{(1)}_{\bar{s}}) \geq \nabla_{\mathbf{v}} Z^{\policy'}_{P^{(1)},\bar{s}}(P^{(1)}_{\bar{s}}) \enskip \forall \policy' \in \policySpace^\star_{(2)},
        \\
        \label{eq:completeness:deriv-min}
        \nabla_{\mathbf{v}} Z^{\policy^\star}_{P^{(2)},\bar{s}}(P^{(2)}_{\bar{s}}) \leq \nabla_{\mathbf{v}} Z^{\policy'}_{P^{(2)},\bar{s}}(P^{(2)}_{\bar{s}}) \enskip \forall \policy' \in \policySpace^\star_{(2)},
    \end{align}
    \end{subequations}
    where the vector $\mathbf{v} \in \RR^{|S|}$ is defined as
    \[
    \mathbf{v} = \begin{cases}
        P^{(2)} - P^{(1)} & \,\, \text{if } \rho^{\policy^\star}_{P^{(2)}} > \rho^{\policy^\star}_{P^{(1)}},
        \\
        P^{(1)} - P^{(2)} & \,\, \text{otherwise.} 
    \end{cases}
    \]
    Then, the policy $\policy^\star$ is \RBE
\end{restatable}

In the proof, presented in \cref{appendix:proof:completeness}, we show that there always exists a policy $\policy^\star \in \policySpace^\star_{(2)}$ that satisfies \cref{eq:completeness:deriv-min,eq:completeness:deriv-max}.
As discussed next, a practical implementation of \cref{thm:char:complete} is to choose $P^{(1)}$ and $P^{(2)}$ as worst- and best-case transition functions.

\begin{algorithm}[t!]
\caption{Computation of \RBE policy.}

\label{alg:compute-rbe}
\begin{algorithmic}[1]
\Statex\textbf{Input:} $s$-rectangular RMDP $(S, A, \cP, r, \gamma)$
\Statex\textbf{Output:} \RBE policy $\pi^\star \in \policySpace^\star_\BE$

\State\label{line:maxmin}$\Pi \gets \argmax_{\pi} \min_{P \in \cP} \rho_P^\pi$ %

\If{$|\Pi| > 1$} \label{line:check-maxmin}
    \State\label{line:maxmax}$\Pi \gets \argmax_{\pi \in \Pi} \max_{P \in \cP} \rho_P^\pi$ %
    \If{$|\Pi| > 1$} \label{line:check-maxmax}
        \State\label{line:policy}$\policy \gets \policySpace$
        \State\label{line:P1}${P}^{(1)} \gets \argmin_{P \in \cP} \rho^{\policy}_P$
        \State\label{line:P2}${P}^{(2)} \gets \argmax_{P \in \cP} \rho^{\policy}_P$
        \State\label{line:refinement1}$\policySpace \gets \argmax_{\policy \in \policySpace} \rho^\policy_{P^{(1)}}$
        \State\label{line:refinement2}$\policySpace \gets \argmax_{\policy \in \policySpace} \rho^\policy_{P^{(2)}}$
        \State\label{line:vector}$\mathbf{v} \gets P^{(2)} - P^{(1)} \enskip \forall\bar{s}$
        \State\label{line:deriv-max}$\Pi \gets \bigtimes_{\bar{s} \in S} \argmax_{\policy(\bar{s}) \in \policySpace} \nabla_{\mathbf{v}} Z^\pi_{P^{(1)},\bar{s}}(P^{(1)}_{\bar{s}})$ %
        
        \If{$|\Pi| > 1$} \label{line:check-deriv-max}
            \State \label{line:deriv-min} $\Pi \gets \bigtimes_{\bar{s} \in S} \argmin_{\policy(\bar{s}) \in \policySpace} \nabla_{\mathbf{v}} Z^\pi_{P^{(2)},\bar{s}}(P^{(2)}_{\bar{s}})$ %
        \EndIf
    \EndIf
\EndIf

\State \Return any $\pi^\star \in \Pi$ %
\end{algorithmic}
\end{algorithm}

\subsection{Algorithm}\label{subsec:alg}
\cref{thm:char:complete} leads to \Cref{alg:compute-rbe} for computing an \RBE policy.
In particular we iteratively refine $\policySpace$ by applying the criteria presented above to obtain an \RBE policy.

We first use robust value iteration to compute the set of \optimalrobust policies (Line~\ref{line:maxmin}), which, if a singleton\footnote{An optimal policy $\policy^\star$ is unique if, for every state $s \in S$, the robust value $V^{\policy^\star}_{\cP}(s)$ is strictly higher than $R(s,a) + \dotp{\gamma P(s,a)}{V^{\policy^\star}_{\cP}}$ for all other actions $a \neq \policy^\star(s)$~\cite{DBLP:books/wi/Puterman94}.
For randomized policies, we instead must check for strict concavity of the value function with respect to the policy, \eg, by deriving the \optimalrobust Bellman operator explicitly as in~\citet{DBLP:conf/icml/KumarWLM24}.
}, consists of an \RBE policy by \cref{corr:char:maxmin_unique}, thus solving \cref{prob:statement}.
Otherwise, we analogously compute the set of optimal policies under the maximizing transition function (Line~\ref{line:maxmax}), which, if a singleton, contains an \RBE policy by \cref{corr:char:maxmax_unique}.

If this set is still not a singleton, we arbitrarily select a policy $\policy$ from the remaining updated set $\policySpace $ (Line~\ref{line:policy}) and compute the minimizing and maximizing transition functions ${P}^{(1)}$ and ${P}^{(2)}$ (Lines~\ref{line:P1} and \ref{line:P2}).
We then refine the policy set by keeping only those that first maximize the expected return for ${P}^{(1)}$ and then for ${P}^{(2)}$ (Lines~\ref{line:refinement1} and \ref{line:refinement2}).
For every $\bar{s} \in S$, we define $\mathbf{v} \gets {P}^{(2)} - {P}^{(1)}$ as per \cref{thm:char:complete} (Line~\ref{line:vector}).
Next, we refine the set of policies by, in every state $\bar{s} \in S$, only selecting actions that \emph{maximize} the directional derivative at the \emph{minimizer} ${P}^{(1)}$ (Line~\ref{line:deriv-max}).
The Cartesian product $\policySpace \gets \bigtimes_{\bar{s} \in S} \cdots$ of these actions gives the set of policies that satisfy \cref{eq:completeness:deriv-max}.
If $\policySpace$ is now a singleton, then it satisfies \cref{corr:char:derivative} and, thus, $\policy^\star \in \policySpace$ is \RBE.
 Otherwise, if multiple policies remain, we perform the analogous refinement---over the set of policies obtained in Line~\ref{line:deriv-max}---to \emph{minimize} the directional derivative at the \emph{maximizing} transition function ${P}^{(2)}$ (Line~\ref{line:deriv-min}).
 
Any returned policy $\policy^\star$ satisfies at least one of the \cref{corr:char:maxmin_unique,corr:char:maxmax_unique,corr:char:derivative} or \cref{thm:char:complete}, thus showing that the algorithm always returns a \RBE policy.

\begin{remark}
We can easily amend \cref{alg:compute-rbe} for a policy that \emph{minimizes} expected return under the \emph{maximizing} probabilities.
In this case, we replace all $\min$ with $\max$ and vice~versa.
We shall see such an application in \cref{sec:implementation}.
\end{remark}

\paragraph{Complexity.}
The computations in \Cref{alg:compute-rbe} lead to a manageable overhead compared to the standard robust value iteration in Line~\ref{line:check-maxmin}.
First, Line~\ref{line:maxmax} amounts to running robust value iteration again, but over a potentially smaller subset of actions per state, increasing complexity by a constant smaller than 2.
Next, Lines~\ref{line:P1} to \ref{line:refinement2} compute the minimizer and maximizer, and solve the two associated MDPs using standard value iteration.
Finally, maximizing the derivatives (Line~\ref{line:deriv-max}) amounts to solving a linear equation system of size $|S|$ for every state and action~\citep{DBLP:conf/vmcai/HeckSJMK22,DBLP:conf/cav/BadingsJMTJ23}.
Solving each equation system has worst-case complexity $O(|S|^3)$, yielding an overall complexity of $O(|S|^4 \cdot |A|)$ for Line~\ref{line:deriv-max} (and, by symmetry, also for Line~\ref{line:deriv-min}).
Thus, whenever computing an optimal robust policy is feasible, the additional overhead of \cref{alg:compute-rbe} is also manageable.

\section{Empirical Evaluation}\label{sec:implementation}
\label{sec:experiment:prism}
In \cref{sec:characterization}, we presented an efficient and complete algorithm for computing \RBE policies.
In this section, we experimentally show the applicability of our algorithm within different implementations of robust value iteration.
Our primary objective is to provide a proof of concept to confirm the theoretical results from \cref{sec:characterization}.
The experiments ran on an Apple MacBook with an M4 Pro chip and 24GB of RAM.
The code is available on {\color{Sepia} \url{https://github.com/tbadings/best-effort-rmdps}}.

\subsection{\BestEffort Policies for Interval MDPs}
We consider robust value iteration within PRISM, a popular tool for MDPs~\cite{DBLP:conf/cav/KwiatkowskaNP11}. %
PRISM only supports \emph{interval MDPs} (IMDPs), \ie $(s,a)$-rectangular RMDPs with interval-valued probabilities.
We consider variants of a \emph{slippery gridworld} IMDP (see \cref{appendix:models} for details).
The objective is to minimize the expected number of steps to reach the goal state.
When the agent slips, it remains in the same state.
The agent can move in each direction with two actions: one where the slipping probability $p$ is \emph{fixed}, and one where it belongs to the \emph{interval} $[q,p]$.
Since the goal is to minimize the number of steps, the worst-case slipping probability is $p$, so the robust value of both action types is the same.
However, only a policy that always picks the interval-valued action is \besteffort.

To show that PRISM returns an arbitrary \optimalrobust (but not necessarily \RBE) policy, we define the IMDP's actions in different orders.
Let $\nu \in [0,1]$ be the fraction of states in which the \besteffort action is defined first.
We consider $\nu=0$ (non-\besteffort always defined first), $\nu=1$ (\besteffort defined first), and $\nu=0.5$ (a coin-flip decides which action is defined first).
We repeat each experiment over $10$ seeds.
The results in \cref{tab:gridworld_prism} show the percentage of states where the optimal robust policy returned by PRISM chooses the \besteffort action (\ie the action with an interval for the slipping probability).
Essentially, the PRISM policy sticks to the first action it finds to be \optimalrobust, so the fraction of \besteffort actions is roughly proportional to $\nu$.
Thus, PRISM finds \optimalrobust policies, but not necessarily \RBE ones.

Conversely, for our method, we apply \cref{corr:char:maxmax_unique} by again running robust value iteration with PRISM, but this time over the \optimalrobust policies and for the \emph{best-case} slipping probability.
This second run of value iteration is over a smaller set of policies and less than doubles the runtime (especially for $|S| = 10^4$), thus confirming our results from \cref{sec:characterization}:
the complexity for computing \RBE policies is still dominated by that of robust value iteration, making the process feasible whenever robust optimal policies can be computed.
The policy obtained using our approach always chooses actions with the interval-valued slipping probabilities.
Thus, and as confirmed by the rightmost column of \cref{tab:gridworld_prism}, the use of \cref{corr:char:maxmax_unique} indeed always leads to \RBE policies.

\begin{table}[t!]
\caption{Comparison to PRISM on the gridworld IMDPs, showing the grid sizes, probability $\nu$ to define the \besteffort action first, runtimes, and percentage of states in which the resulting optimal policy chooses a \besteffort (\BE) action.}
    \centering
    \small
\begin{tabular}{@{}llrrrr@{}}
\toprule
& & \multicolumn{2}{c}{{PRISM}} & \multicolumn{2}{c}{{+ Best-case (Corr.~\ref{corr:char:maxmax_unique})}}
\\ \cmidrule(lr){3-4} \cmidrule(lr){5-6}
$|S|$ & $\nu$ & Time [s] & \BE~[\%] & Time [s] & \BE~[\%] \\
\midrule
\multirow[t]{3}{*}{100}%
& $0.0$ & $2.0$ & $21.9$ & $3.9$ & $100.0$ \\
 & $0.5$ & $1.9$ & $59.6$ & $3.8$ & $100.0$ \\
 & $1.0$ & $1.9$ & $89.9$ & $3.9$ & $100.0$ \\
\midrule
\multirow[t]{3}{*}{$900$}%
& $0.0$ & $2.1$ & $23.3$ & $4.0$ & $100.0$ \\
 & $0.5$ & $2.1$ & $62.0$ & $4.1$ & $100.0$ \\
 & $1.0$ & $2.1$ & $87.4$ & $4.2$ & $100.0$ \\
\midrule
\multirow[t]{3}{*}{$10\,000$}%
& $0.0$ & $48.9$ & $21.2$ & $54.4$ & $100.0$ \\
 & $0.5$ & $54.9$ & $39.3$ & $61.2$ & $100.0$ \\
 & $1.0$ & $51.0$ & $85.4$ & $56.5$ & $100.0$ \\
\bottomrule
\end{tabular}

\label{tab:gridworld_prism}
\end{table}

\subsection{\BestEffort Policies for $s$-Rectangular RMDPs}
\label{sec:experiment:rvi}
To show the applicability of our methods beyond IMDPs, we create a basic implementation of robust value iteration and the derivative computation for $s$-rectangular RMDPs (see \cref{appendix:models} for details).
We consider variants of the same slippery gridworld as in \cref{sec:experiment:prism} but now with an $s$-rectangular uncertainty set.
For this RMDP, either \cref{corr:char:maxmax_unique} or \ref{corr:char:derivative} is sufficient to obtain an \RBE policy.
Therefore, instead of implementing \cref{alg:compute-rbe} sequentially, we test both separately on top of robust value iteration.

The results in \cref{tab:gridworld_rvi} give the same picture as in \cref{sec:experiment:prism}: if multiple \optimalrobust policies exist, robust value iteration returns the first optimal actions it finds.
By contrast, our methods provide simple yet effective tie-break rules, either by returning a policy that is also optimal under the best-case transition probabilities (RVI + \cref{corr:char:maxmax_unique}), or by returning a policy with the highest derivatives (RVI + \cref{corr:char:derivative}).
The former less than doubles the total runtime (especially for the larger models), while computing derivatives is even cheaper, increasing the total runtime by less than 10\%.

\begin{table}[t!]
\caption{Results on the gridworld RMDPs, for robust value iteration (RVI), RVI plus optimizing for the best-case probabilities, and RVI plus optimizing for the derivatives.}
    \centering
    \setlength{\tabcolsep}{2pt}
\small
\begin{tabular}{@{}llrrrrrr@{}}
\toprule
& & \multicolumn{2}{c}{{RVI}} & \multicolumn{2}{c}{{+ Best-case (Corr.~\ref{corr:char:maxmax_unique})}} & \multicolumn{2}{c}{{+ Deriv. (Corr.~\ref{corr:char:derivative})}}
\\ \cmidrule(lr){3-4} \cmidrule(lr){5-6} \cmidrule(lr){7-8}
$|S|$ & $\nu$ & Time [s] & \BE~[\%] & Time [s] & \BE~[\%] & Time [s] & \BE~[\%] \\
\midrule
\multirow[t]{3}{*}{$100$}%
& $0.0$\!\! & $7.0$ & $0.0$ & $11.7$ & $100.0$ & $7.1$ & $100.0$ \\
 & $0.5$\!\! & $7.0$ & $49.2$ & $11.7$ & $100.0$ & $7.1$ & $100.0$ \\
 & $1.0$\!\! & $7.6$ & $100.0$ & $12.7$ & $100.0$ & $7.6$ & $100.0$ \\
\midrule
\multirow[t]{3}{*}{$400$}%
& $0.0$\!\! & $49.5$ & $0.0$ & $83.4$ & $100.0$ & $50.1$ & $100.0$ \\
 & $0.5$\!\! & $50.1$ & $48.0$ & $84.4$ & $100.0$ & $50.6$ & $100.0$ \\
 & $1.0$\!\! & $48.4$ & $100.0$ & $81.7$ & $100.0$ & $48.9$ & $100.0$ \\
\midrule
\multirow[t]{3}{*}{$900$}%
& $0.0$\!\! & $163.6$ & $0.0$ & $274.4$ & $100.0$ & $172.0$ & $100.0$ \\
 & $0.5$\!\! & $163.4$ & $50.1$ & $273.8$ & $100.0$ & $171.9$ & $100.0$ \\
 & $1.0$\!\! & $164.1$ & $100.0$ & $275.0$ & $100.0$ & $172.6$ & $100.0$ \\
\bottomrule
\end{tabular}

\label{tab:gridworld_rvi}
\end{table}

\section{Related Work}
\label{sec:related}

The notion of \besteffort was first introduced in a game theoretic context by~\citet{faella2009} as a relaxation of ``winning'' policies (or strategies).
These ideas have been adapted to \emph{reactive synthesis}, where in the absence of a winning strategy, \besteffort policies can be computed at the same cost~\citep{Aminof_DeGiacomo_Murano_Rubin_2019,ijcai2020p232,de2025a}.
Closest to our work are~\citet{10175747} and~\citet{10.1007/978-3-031-56940-1_17}, who study \besteffort for stochastic games where each transition probability is only constrained to lie within the open interval $(0,1)$.
Crucially,~\citet{10175747,10.1007/978-3-031-56940-1_17} exploit this lack of probability bounds to construct a three-valued abstraction of policies (\emph{winning}, \emph{losing}, and \emph{pending}) which is central to their characterization of \besteffort policies.
However, this does not carry over to RMDPs, where probabilities are bounded subsets of $[0,1]$, thus breaking a direct translation of their characterization to the RMDP setting.

Related are \emph{lexicographic orderings} over objectives for MDPs~\citep{DBLP:conf/aaai/WrayZM15} and algorithms for stochastic games that progressively prune suboptimal actions per objective~\citep{Chatterjee2024}.
While our algorithm is conceptually similar, the refinement to \besteffort policies requires different reasoning over the dominance order over policies.
In \emph{multi-objective MDPs} (MOMDPs), multiple objectives are combined, leading to Pareto optimality~\citep{10.1007/978-3-030-45190-5_19,10.1007/978-3-540-71209-1_6}.
While MOMDPs require a trade-off between the objectives, our setting uses \besteffort as a hard refinement within the \optimalrobust policies.
Finally, weakly related are partial orders over states of MDPs~\citep{10.1007/978-3-319-89366-2_20} and monotonicity in parametric Markov chains~\citep{DBLP:conf/atva/SpelJK19}.

While we focus on $s$-rectangular RMDPs, our definitions of \besteffort and dominance carry over to other models, such as $k$- or non-rectangular RMDPs~\cite{DBLP:conf/icml/MannorSST04,DBLP:journals/mor/GoyalG23,DBLP:conf/aaai/GadotDKELM24} and parametric MDPs~\cite{DBLP:conf/atva/QuatmannD0JK16}.
However, computing optimal policies for these models is much harder---up to NP-hard for general non-rectangular RMDPs~\cite{DBLP:journals/mor/WiesemannKR13}.
Thus, adapting dynamic programming methods to these models is still an open problem.

\section{Conclusion}\label{sec:conclusions}

We presented a principled tie-breaker among \optimalrobust policies in RMDPs based on \besteffort.
Our proposed \RBE policies maximize the worst-case expected return but also achieve a maximal expected return under non-adversarial transition probabilities.
We fully characterized \RBE policies and presented an algorithm for computing them.
Our experiments showed how to use our methods as an effective and efficient tie-breaker within robust value iteration.

Future work includes generalizing our methods to non-rectangular RMDPs or parametric MDPs.
Moreover, our methods still rely on first computing a policy under adversarial transition probabilities.
A next step is to consider $\varepsilon$-close \optimalrobust policies and optimize for \besteffort within this broader context.
Finally, we aim to study settings with a Bayesian prior over the uncertainty set~\citep{murphy2001introduction}.

\section{Acknowledgments}
This research is supported by the EPSRC grant EP/Y028872/1, \emph{Mathematical Foundations of Intelligence: An ``Erlangen Programme'' for AI}.
\bibliography{bibfile_dblp}

\clearpage
\appendix

\section{Proofs}
\label{appendix:proofs}

This appendix provides the complete proofs of the theoretical results presented in the main paper.
For clarity and self-containment, all theorems and lemmas are restated before their corresponding proofs.

\subsection{Existence of \BestEffort Policies}
\label{app:ex:proofs}

This section contains the proof of \Cref{th:rmdp:existence:robustbesteffort}.
We first prove that the set of \besteffort policies in an \srectangular RMDP is non-empty (\Cref{th:rmdp:existence:besteffort}) and then show that its intersection with the set of \optimalrobust is non-empty as well.

\begin{restatable}[Existence of \besteffort policies]{theorem}{beexistence}
\label{th:rmdp:existence:besteffort}
The set of \emph{\besteffort policies} $\policySpace_\BE$ is nonempty.
\end{restatable}

\begin{proof}
To prove the existence of at least one \besteffort policy in an RMDP $\mathcal{M}$, we show the existence of a maximal element (with respect to the dominance order) in the set of policies $\policySpace$ in $\mathcal{M}$, \ie a \besteffort policy.

We start by showing that the value functions induced by policies are smooth rational functions over the choice of transition probabilities~\citep[Theorem~6.1.1]{DBLP:books/wi/Puterman94}, and are, therefore, well-defined and comparable over the uncertainty set.
The expected return $V^\policy_P \in \RR^{|S|}$ under the policy $\policy \in \policySpace$ and transition function $P \in \cP$ is written in matrix form as
\begin{equation}
    \label{eq:thm:existence:proof1}
    V^\policy_P = (I - \gamma P^\policy)^{-1} R^\policy,
\end{equation}
where $P^\policy \in \RR^{|S| \times |S|}$ and $R^\policy \in \RR^{|S|}$ are the matrix and vector forms of \cref{eq:P_policy,eq:R_policy} over all states, respectively.
Thus, the expected return $\rho^\policy_P = \langle \initState, V^\policy_P \rangle$ for a fixed policy $\policy$ is a smooth rational function over the transition function $P \in \cP$, which is called the \emph{solution function}.\footnote{For details on such solution functions for parametric MDPs, we refer to~\citet{DBLP:journals/fmsd/JungesAHJKQV24}.}
Note that, since the solution functions induced by policies are smooth, any strict improvement occurs over a nontrivial (non-measure-zero) subset of the domain.

Recall that the dominance relation $\geq_{\cP}$ in \cref{def:rmdp:strict_dominance} is defined as follows:
\begin{center}
    \begin{minipage}{0.95\columnwidth}%
        \emph{
  Let $\policy, \policy' \in \policySpace$ be policies for the RMDP $\rmdp$.
    The policy $\policy$ \emph{dominates} $\policy'$, written $\policy \geq_{\cP} \policy'$, if and only if $\rho^\policy_P \geq \rho^{\policy'}_P$ for all $P \in \cP$.
        }
    \end{minipage}%
\end{center}

This relation is reflexive, transitive, and antisymmetric.
Thus, under this ordering, the set of policies $\policySpace$ becomes a partially ordered set.
A policy is \besteffort if it is a maximal element in this partially ordered set, \ie there is no other policy that dominates it over all $P \in \cP$.

To apply Zorn’s Lemma~\citep{Halmos1974} and establish the existence of a maximal element, we show that every chain (\ie a totally ordered subset) $C \subseteq \policySpace$ has an upper bound.
As mentioned, each policy $\policy \in C$ induces an expected return of $\rho^\policy_P = \langle \initState, V^\policy_P \rangle$, which is a smooth rational function of $P \in \cP$.
Because the policies in $C$ are totally ordered, their corresponding returns $\rho^\policy_{P}$ form a \emph{pointwise increasing} chain: for any $\policy, \policy' \in C$, either $\rho^\policy_P \geq \rho^{\policy'}_P$ or vice versa for all $P \in \cP$.

Due to the smoothness and pointwise comparability of these functions, and the compactness of the uncertainty set $\cP$, we can define the pointwise supremum $\rho^{\bar{\policy}}_P$ of the chain as the return value of a policy $\bar{\policy} \in \argmax_{\policy \in C} \rho^{\policy}_P$.
The policy space $\policySpace$ is the set of all probability distributions over the finite set of actions $A$ and is, hence, compact.
Moreover, since the mapping $\policy \mapsto \rho^\policy_P$ is continuous for each fixed $P$, it follows that the policy $\bar{\policy}$ realizing this supremum exists in $\policySpace$, \ie $\bar{\policy} \in \policySpace$.
That is, $\bar{\policy}$ realizes the pointwise supremum and thus serves as an upper bound of the chain $C$ in $\policySpace$.
Thus, every chain has an upper bound.

Because the policy space is nonempty and every chain has an upper bound, Zorn’s Lemma~\citep{Halmos1974} guarantees the existence of at least one maximal element with respect to dominance.
By \Cref{def:rmdp:besteffort}, these maximal elements are precisely the \besteffort policies.
Hence, $\policySpace_\BE$ is nonempty.
\end{proof}

\orbeexistence*
\begin{proof}
By definition, an \optimalrobust policy $\policy^{\star} \in \policySpace^\star$, which always exists by construction, maximizes the expected return under the worst-case (fully adversarial) transition function.
This means no other policy can strictly dominate $\policy^{\star}$ in that adversarial transition function, that is, $\nexists\ \policy \in \policySpace \setminus \{\policy^\star\} \text{ s.t. } \rho^\policy_\cP > \rho^{\policy^{\star}}_\cP$ where $\rho^\policy_\cP$ and $\rho^{\policy^{\star}}_\cP$ represents the robust expected return (as defined in \Cref{eq:RMDP:robust_return}) for policy $\policy$ and $\policy^{\star}$, respectively.

Moreover, when this other policy $\policy$ is non-\optimalrobust, \ie $\pi \in \policySpace \setminus \policySpace^\star$, we can also deduce that $\policy^{\star}$ cannot be strictly dominated by $\policy$.
More precisely, since $\rho^{\policy^\star}_\cP = \min_{P \in \cP} \rho^\policy_P$, it holds that $\rho^{\policy^\star}_{P'} \geq \rho^{\policy^\star}_\cP$ for all $P' \in \cP$.
By letting $P' \in \argmin_{P \in \cP} \rho^{\policy}_P$, we thus obtain $\rho^{\policy^\star}_{P'} \geq \rho^{\policy^\star}_\cP > \rho^{\policy}_\cP = \rho^{\policy}_{P'}$, which implies that $\policy \not>_{\cP} \policy^\star$.

Then, two cases arise:

\begin{enumerate}
  \item For all policies $\policy \in \Pi$, $\rho^\policy_\cP < \rho^{\policy^{\star}}_\cP$.
  In this case, $\policy^{\star}$ is the unique \optimalrobust policy and is trivially \besteffort as it is a maximal element under strict dominance, and thus cannot be strictly dominated by any other policy (as detailed above).
  Hence, $\policy^{\star} \in \policySpace_\BE$ and so $\policySpace^\star \cap \policySpace_\BE \neq \emptyset$.

  \item There exists at least one other policy $\policy' \in \policySpace^\star \setminus \{\policy^\star\}$ such that $\rho^{\policy{'}}_\cP = \rho^{\policy^{\star}}_\cP$.
  Here, the set of \optimalrobust policies is not a singleton.
  Among these policies, there must exist at least one \besteffort policy.
  This follows from the same reasoning based on partial orders and maximal elements as in the proof of \Cref{th:rmdp:existence:besteffort}.
  In particular, if no \besteffort policy existed, the set of policies ordered by $\geq_{\cP}$ would have no maximal element, contradicting the existence of such elements ensured by Zorn’s Lemma.
  Thus, the intersection is nonempty.
\end{enumerate}

In either case, there exists at least one policy that is both \optimalrobust and \besteffort.
\end{proof}

\Cref{fig:rbe:existance} summarizes this, showing that the orange region (\optimalrobust policies), the blue region (\besteffort policies), and their overlapping area (\RBE policies) are all nonempty.

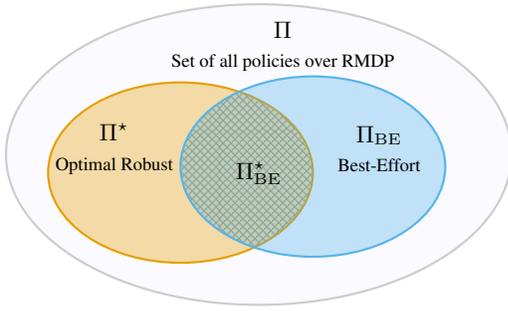
\begin{figure}[t!]
\centering
    \begin{tikzpicture}[scale=0.8]

\definecolor{cborange}{RGB}{230,159,0}
\definecolor{cbsky}{RGB}{86,180,233}

\draw[thick, fill=blue!10, opacity=0.2] (0,0) ellipse (4.2cm and 2.5cm);
\node[align=center]  at (0.4,1.8) {\small $\policySpace$\\\scriptsize Set of all policies over RMDP};

\begin{scope}
  \clip (-1.3,-0.3) ellipse (2.2cm and 1.5cm);
  \fill[pattern=crosshatch, pattern color=black!40]
        (0.9,-0.2) ellipse (2.2cm and 1.5cm);
\end{scope}

\draw[thick, cborange] (-1.3,-0.3) ellipse (2.2cm and 1.5cm);
\fill[cborange, opacity=0.35] (-1.3,-0.3) ellipse (2.2cm and 1.5cm);
\node[align=center] at (-2.4,0.1) 
  {\small $\policySpace^\star$\\\scriptsize \OptimalRobust};

\draw[thick, cbsky] (0.9,-0.2) ellipse (2.2cm and 1.5cm);
\fill[cbsky, opacity=0.35] (0.9,-0.2) ellipse (2.2cm and 1.5cm);
\node[align=center] at (2,0.1) 
  {\small $\policySpace_\BE$\\\scriptsize \BestEffort};

\node[align=center] at (0,-0.3)
  {\small $\policySpace^\star_\BE$};

\end{tikzpicture}
\caption{Structure of the policy space in an RMDP.
The gray ellipse represents the set of all policies admissible in the RMDP.
The orange region denotes the set of \emph{\optimalrobust} ($\policySpace^\star$), while the blue region indicates the set of \emph{\besteffort policies} ($\policySpace^\star_\BE$).
The area where the two regions overlap corresponds to the \RBE policies ($\policySpace^\star \cap \policySpace_\BE = \policySpace^\star_\BE$).}
\label{fig:rbe:existance}
\end{figure}

\subsection{Proofs of \Cref{corr:char:maxmin_unique,corr:char:maxmax_unique,corr:char:derivative}}\label{app:characterization:proofs}
For completeness, we provide the proofs of \Cref{corr:char:maxmin_unique,corr:char:maxmax_unique,corr:char:derivative}, presented in \Cref{sec:characterization}.

\maxminUnique*
\begin{proof}
    The proof follows immediately from \cref{th:rmdp:existence:robustbesteffort}, which states that the set of \optimalrobust policies $\policySpace^\star$ always contains a \besteffort policy.
    Thus, if there is a single \optimalrobust policy, this policy must also be \besteffort.
\end{proof}

\maxmaxUnique*
\begin{proof}
    First, \cref{th:rmdp:existence:robustbesteffort} states that there exists an \RBE policy within $\check\policySpace^\star$.
    Second, the uniqueness of $\policy^\star \in \hat\policySpace^\star$ implies there exists $P \in \cP$ such that $\rho^{\policy^\star}_P > \rho^{\policy'}_P$ for all $\policy' \in \check\policySpace^\star$.
    Thus, there is no policy in $\check\policySpace^\star$ that strictly dominates $\policy^\star$, which proves that $\policy^\star$ is \RBE.
\end{proof}

\derivative*
\begin{proof}
    We will show that, for the policy $\policy^\star$, there exists a $P \in \cP$ such that $\rho^{\policy^\star}_P > \rho^{\policy'}_P$ for all $\policy' \in \policySpace^\star \setminus \{\policy^\star\}$, and thus, $\policy^\star$ is \RBE by \cref{thm:ORBE_char}.

    First, suppose that $\policy' \in \policySpace^\star \setminus (\policySpace^\star_{(1)} \cup \{\policy^\star\})$.
    As $\policy'$ is not in $\policySpace^\star_{(1)} = \argmax_{\policy \in \policySpace^\star} \rho^\policy_{P^\star}$, it holds that $\rho^{\policy^\star}_{P^\star} > \rho^{\policy'}_{P^\star}$.
    Thus, $\policy^\star$ cannot be strictly dominated by $\policy'$.

    On the other hand, suppose that $\policy' \in \policySpace^\star_{(1)} \setminus \{\policy^\star\}$.
    In this case, it holds that $\rho^{\policy^\star}_{P^\star} = \rho^{\policy'}_{P\star}$.
    Because the expected return is a smooth function and $\mathbf{v}$ points inside the uncertainty set, the condition $\nabla_{\mathbf{v}} Z^{\policy^\star}_{P^\star,\bar{s}}(P^\star_{\bar{s}}) > \nabla_{\mathbf{v}} Z^{\policy'}_{P^\star,\bar{s}}(P^\star_{\bar{s}})$ implies that there exists $\lambda > 0$ such that 
    \[
    Z^{\policy^\star}_{P^\star,\bar{s}}(P^\star_{\bar{s}} + \lambda \mathbf{v}) > Z^{\policy'}_{P^\star,\bar{s}}(P^\star_{\bar{s}} + \lambda \mathbf{v}).
    \]
    As this condition holds for every state $\bar{s} \in \States$ and the policies $\policy^\star$ and $\policy'$ must differ in at least one state, it follows that, also in the second case, $\policy'$ cannot strictly dominate $\policy^\star$.
    Thus, $\policy^\star$ cannot be strictly dominated by any policy $\policy' \in \policySpace^\star \setminus \{\policy^\star\}$, so we conclude that $\policy^\star$ is \RBE by \cref{thm:ORBE_char}.
\end{proof}

\subsection{Rational form of the value function.}
\label{appendix:proof:rational}

Toward the proof of \cref{thm:char:complete}, we show in the following \cref{lemma:rational} that the parametric transition function $Z^\policy_{P,\bar{s}}(P_{\bar{s}})$ in \cref{def:parametric_V} is a rational function of degree one (\ie a fraction of two linear functions).
In the proof \cref{lemma:rational}, we use the \emph{Sherman-Morrison formula}, a well-known matrix identity~\cite[Sect.~2.1.4]{Golub2013MatrixComputations}, which states that, for a nonsingular matrix $A \in \RR^{n \times n}$, vectors $u, v \in \RR^n$, and $1 + v^\top A^{-1} u \neq 0$, it holds that
\begin{equation}
    \label{eq:ShermanMorrison}
    (A + uv^\top)^{-1} = A^{-1} - \frac{A^{-1} uv^\top A^{-1}}{1 + v^\top A^{-1} u}.
\end{equation}
In other words, the inverse of a rank-one update to a nonsingular matrix $A$ can be expressed in terms of the inverse of $A$ itself, as long as $1 + v^\top A^{-1} u \neq 0$.
For a derivation of \cref{eq:ShermanMorrison}, we refer to~\citet[Sect.~2.1.4]{Golub2013MatrixComputations}.

Recall that $P_{\bar{s}} \colon A \to \dist{S}$ is a function from actions to distributions over states.
For notational simplicity, we will also interpret $P_{\bar{s}}(a)$, $a \in A$, as a vector in $\RR^{|S|}$, and $P_{\bar{s}}$ as a matrix in $\RR^{|A| \times |S|}$.

\begin{restatable}[Value function as rational]{lemma}{rational}
    \label{lemma:rational}
	For any policy $\policy$, state $\bar{s} \in S$, and partial transition function $P_{-\bar{s}}$, the value $Z^\policy_{P,\bar{s}}(P_{\bar{s}})$ in state $\bar{s}$ as a function of the completion $P_{\bar{s}} \in \cP_{\bar{s}}$ can be written as a rational function of the form
	\begin{equation}
		\label{eq:lemma:rational}
		Z^\policy_{P,\bar{s}}(P_{\bar{s}}) = \frac{R^\policy(\bar{s}) + \sum_{a \in A} \alpha_a^\top P_{\bar{s}}(a) }{1 - \sum_{a \in A} \varphi_a^\top P_{\bar{s}}(a)},
	\end{equation}
        where $R^\policy$ is defined by \cref{eq:R_policy}, and for all $a \in A$, the coefficients $\alpha_a \in \RR_{\geq 0}^{|S|}$ and $\varphi_a \in \RR_{\geq 0}^{|S|}$ are defined appropriately.
\end{restatable}

\begin{proof}
    The expected return $V^\policy_P \in \RR^{|S|}$ under the policy $\policy$ and transition function $P$ is written in matrix form as
    \begin{equation}
        \label{eq:lemma:rational:proof1}
        V^\policy_P = (I - \gamma P^\policy)^{-1} R^\policy,
    \end{equation}
    where $P^\policy \in \RR^{|S| \times |S|}$ and $R^\policy \in \RR^{|S|}$ are the matrix and vector forms of \cref{eq:P_policy,eq:R_policy} over all states, respectively.
    Define $e_{s} \in \RR^{|S|}$ as the vector with value $1$ in entry $s$ only and $0$ otherwise, and conversely, define $e_{\neg s} \in \RR^{|S|}$ as the vector with value $0$ in entry $s$ only and $1$ otherwise.
    Then, \cref{eq:lemma:rational:proof1} can be decomposed as
    \begin{equation}
        \label{eq:lemma:rational:proof2}
        V^\policy_{P} = \left(I 
        - \gamma \cdot \text{diag}(e_{\neg \bar{s}}) P^\policy 
        - \gamma \cdot \text{diag}(e_{\bar{s}}) P^\policy
        \right)^{-1} R^\policy.
    \end{equation}
    For \cref{lemma:rational}, we are given a partial transition function $P_{-\bar{s}}$, which thus fixes $P^\policy(s)$ is fixed for all $s \in \States \setminus \{\bar{s}\}$.
    Hence, by interpreting the completion $P_{\bar{s}}$ as a vector in $\RR^{|S|}$, we can rewrite \cref{eq:lemma:rational:proof2} as
    \begin{equation}
        \label{eq:lemma:rational:proof3}
        V^\policy_{P} = \left(G - \gamma \cdot e_{\bar{s}} P_{\bar{s}}^\top \right)^{-1} R^\policy,
    \end{equation}
    where $G = I - \gamma \cdot \text{diag}(e_{\neg \bar{s}}) P^\policy$ is fixed.
    As $\gamma < 1$, the matrix $G$ is nonsingular, so we can apply the Sherman-Morrison formula from \cref{eq:ShermanMorrison} (with $A \coloneqq G$, $u \coloneqq -\gamma \cdot e_{\bar{s}}$ and $v \coloneqq P_{\bar{s}})$ to obtain
    \begin{align*}
        &\left(G - \gamma \cdot e_{\bar{s}} P_{\bar{s}}^\top \right)^{-1}
        =
        G^{-1} + \frac{G^{-1} (\gamma \cdot e_{\bar{s}} P_{\bar{s}}^\top) G^{-1}}{1 - P_{\bar{s}}^\top G^{-1} \gamma \cdot e_{\bar{s}}},
    \end{align*}
    where $P_{\bar{s}}^\top G^{-1} \gamma \cdot e_{\bar{s}} \in [0,1)$ for $\gamma < 1$ (which we consider; see~\cref{def:MDP}).
    Thus, $V^\policy_P$ in \cref{eq:lemma:rational:proof3} is rewritten as
    \begin{align}
        \label{eq:lemma:rational:proof4}
        V^\policy_P &= \left(
            G^{-1} + \frac{G^{-1} (\gamma \cdot e_{\bar{s}} P_{\bar{s}}^\top) G^{-1}}{1 - P_{\bar{s}}^\top G^{-1} \gamma \cdot e_{\bar{s}}}
        \right) R^\policy
        \\
        \nonumber
        &=
        \left(
        \frac{
            G^{-1} \big[ 1 - P_{\bar{s}}^\top G^{-1} \gamma \cdot e_{\bar{s}} + (\gamma \cdot e_{\bar{s}} P_{\bar{s}}^\top) G^{-1} \big]
        }
        {1 - P_{\bar{s}}^\top G^{-1} \gamma \cdot e_{\bar{s}}}
        \right) R^\policy,
    \end{align}
    which has a numerator and a denominator that are both linear in $P_{\bar{s}}$. Thus, \cref{eq:lemma:rational:proof4} can be written as a rational function between two linear functions.
    Let us write $V^\policy_P = [V^\policy_P(s_1), \ldots,V^\policy_P(s_{|\States|})]$.
    By \cref{def:parametric_V}, we have that $Z^\policy_{P,\bar{s}}(P_{\bar{s}}) = V^\policy_{P_{-\bar{s}} \times P_{\bar{s}}}(\bar{s}) = V^\policy_P(\bar{s})$, yielding the rational form in \cref{eq:lemma:rational} with appropriate coefficients $\alpha_a$ and $\varphi_a$ for all $a \in \Actions$.
    Finally, the domains of $\alpha_a$ and $\varphi_a$ follow from the fact that the rewards and transition probabilities are nonnegative.
\end{proof}

Intuitively, the numerator $R^\policy(\bar{s}) + \sum_{a \in A} \alpha_a^\top P_{\bar{s}}(a)$ in \cref{eq:lemma:rational} is the sum of the immediate reward and the future discounted reward along paths that do \emph{not} loop back to state $\bar{s}$.
Moreover, the term $\sum_{a \in A} \varphi_a^\top P_{\bar{s}}(a)$ in the denominator is the discounted probability of (eventually) looping back to state $\bar{s}$ (and thus acts as a normalizing constant).
For any $\gamma < 1$, this probability is strictly smaller than one.

\subsection{Proof of \Cref{thm:char:complete}}
\label{appendix:proof:completeness}
In this section, we use \cref{lemma:rational} to provide the proof of \Cref{thm:char:complete}, which states that, for any RMDP, our characterization yields an \RBE policy.

\subsubsection{Equivalence of policies.}
First, we show that if two policies attain the same expected return and derivatives under two distinct transition functions, then these policies attain the same expected return on an entire line segment in the space of transition functions.
This intuition is formalized by \Cref{lemma:equivalence_on_line}. 

\begin{restatable}[Equivalence along $P_{\bar{s}}$ line segment]{lemma}{equivalenceOnLine}
    \label{lemma:equivalence_on_line}
    Let $\policy, \policy' \in \policySpace$ be two policies, let $\bar{s} \in S$ be a state, let $P \in \cP$ be a transition function, and let $P^{(i)}_{\bar{s}} \in \cP_{\bar{s}}$, $i=1,2$ be two distinct transition functions in state $\bar{s}$.
    Define $\mathbf{v} = P^{(2)}_{\bar{s}} - P^{(1)}_{\bar{s}}$.
    If, for all $i=1,2$, it holds that
    \begin{subequations}
    \label{eq:equivalence_on_line}
    \begin{align}
        Z^{\policy}_{P,\bar{s}}(P^{(i)}_{\bar{s}}) &= Z^{\policy'}_{P,\bar{s}}(P^{(i)}_{\bar{s}}),
        \\
        \nabla_\mathbf{v} Z^{\policy}_{P,\bar{s}}(P^{(i)}_{\bar{s}}) &= \nabla_\mathbf{v} Z^{\policy'}_{P,\bar{s}}(P^{(i)}_{\bar{s}}),
    \end{align}
    \end{subequations}
    then the expected returns are the same on the entire line segment between points $P^{(1)}_{\bar{s}}$ and $P^{(2)}_{\bar{s}}$, \ie for all $\lambda \in [0,1]$,
    \begin{equation}
        \label{eq:equivalence_on_line2}
        Z^{\policy}_{P,\bar{s}}( q ) = Z^{\policy'}_{P,\bar{s}}( q ) \quad \forall q = \lambda P^{(1)}_{\bar{s}} + (1-\lambda) P^{(2)}_{\bar{s}}.
    \end{equation}
\end{restatable}
\begin{proof}
    The conditions above state that there exist two policies $\policy$ and $\policy'$ that attain the same values in two different points $P_{\bar{s}}^{(1)}$ and $P_{\bar{s}}^{(2)}$.
    Furthermore, the directional derivatives $\nabla_\mathbf{v} Z^{\policy}_{P,\bar{s}}(P^{(i)}_{\bar{s}})$ and $\nabla_\mathbf{v} Z^{\policy'}_{P,\bar{s}}(P^{(i)}_{\bar{s}})$, $i=1,2$, in the direction $\mathbf{v}$ of the line segment connecting $P_{\bar{s}}^{(1)}$ and $P_{\bar{s}}^{(2)}$ are also equal.
    We will show that these four constraints are sufficient for \cref{eq:equivalence_on_line2} to hold.    

    Recall from \cref{lemma:rational} that for any policy $\policy$, the value function $Z^\policy_{P,\bar{s}}(P_{\bar{s}})$ in state $\bar{s}$ is a rational function of the form
    \[
    Z^\policy_{P,\bar{s}}(P_{\bar{s}}) = \frac{R^\policy(\bar{s}) + \sum_{a \in A} \alpha_a^\top P_{\bar{s}}(a) }{1 - \sum_{a \in A} \varphi_a^\top P_{\bar{s}}(a)},
    \]
    with appropriate coefficients $\alpha_a$ and $\varphi_a$, $a \in \Actions$.
    For the proof of \cref{lemma:equivalence_on_line}, we only need to consider the values of $Z^\policy_{P,\bar{s}}(P_{\bar{s}})$ for values $P_{\bar{s}} \in \{ \lambda P^{(1)}_{\bar{s}} + (1-\lambda) P^{(2)}_{\bar{s}} : \lambda \in [0,1] \}$.
    In other words, we restrict the value function to the line segment between $P^{(1)}_{\bar{s}}$ and $P^{(2)}_{\bar{s}}$.
    Thus, we may further simplify the multivariable function $Z^\policy_{P,\bar{s}}(P_{\bar{s}})$ as the univariate function $Y^\policy_{P,\bar{s}} \colon [0,1] \to \RR$ defined for all $\lambda \in [0,1]$ as
    \begin{equation}
        Y^\policy_{P,\bar{s}}(\lambda) 
        = Z^\policy_{P,\bar{s}}( \lambda P^{(1)}_{\bar{s}} + (1-\lambda) P^{(2)}_{\bar{s}} ) 
        = \frac{\tilde{a} + \tilde{b} \lambda}{\tilde{c} + \tilde{d} \lambda},
    \end{equation}
    with appropriate coefficients $\tilde{a} \in \RR$, $\tilde{b} \in \RR$, $\tilde{c} \in \RR$, $\tilde{d} \in \RR$.

    Thus, the value function on the line segment between $P^{(1)}_{\bar{s}}$ and $P^{(2)}_{\bar{s}}$ is defined by four parameters.
    At the same time, we have four constraints on $Y^\policy_{P,\bar{s}}$, given by
    \begin{alignat*}{2}
    Y^\policy_{P,\bar{s}}(0) &= \mu_1, \qquad
    \frac{\partial Y^\policy_{P,\bar{s}}(0)}{\partial\lambda} &&= \delta_1,
    \\
    Y^\policy_{P,\bar{s}}(1) &= \mu_2, \qquad
    \frac{\partial Y^\policy_{P,\bar{s}}(1)}{\partial\lambda} &&= \delta_2,
    \end{alignat*}
    which together fully define these four coefficients of $Y^\policy_{P,\bar{s}}$.\footnote{In fact, the function $Y^\policy_{P,\bar{s}}$ is invariant to joint scaling of the parameters and is, thus, already uniquely defined by three constraints.}
    Therefore, any two policies $\policy$ and $\policy'$ satisfying the conditions in \cref{eq:equivalence_on_line} must lead to the same coefficients $\tilde{a} \in \RR$, $\tilde{b} \in \RR$, $\tilde{c} \in \RR$, $\tilde{d} \in \RR$, and thus satisfy \cref{eq:equivalence_on_line2}.  
\end{proof}

We will use \cref{lemma:equivalence_on_line} to investigate the geometry of the value functions $Z^{\policy}_{P,\bar{s}}$ and $Z^{\policy'}_{P,\bar{s}}$.
For convenience, we simplify the rational function $Z^\policy_{P,\bar{s}}(P_{\bar{s}})$ defined by \cref{eq:lemma:rational} as
\begin{equation}
\label{eq:simplified_rational}
f^\policy_{P,\bar{s}}(x) = \frac{\ba + \bb^\top x}{1 - \bc^\top x},
\end{equation}
where $x \in \RR^{|A| \cdot |S|}$ represents $P_{\bar{s}}(a)$ concatenated for all $a \in A$, and with the appropriate coefficients $\ba \in \RR$, $\bb \in \RR^{|A| \cdot |S|}$, and $\bc \in \RR^{|A| \cdot |S|}$.
Using this notation, let $L^{\policy,\policy'}_{P,\bar{s}}(x)$ denote the difference between the rationals of two policies $\policy$ and $\policy$, \ie
\begin{align}
    \label{eq:rational_difference}
    & {} L^{\policy,\policy'}_{P,\bar{s}}(x) 
    = \frac{\ba + \bb^\top x}{1 - \bc^\top x} - \frac{\ba' + (\bb')^\top x}{1 - (\bc')^\top x}
    \\
    &\!\!\!= \frac{x^\top \diag{\bb' \bc - \bb \bc'} x + (\ba' \bc - \ba \bc' + \bb - \bb')^\top x + \ba - \ba'}{(1 - \bc^\top x)(1 - (\bc')^\top x)},
    \nonumber
\end{align}
where the vector-vector multiplication is element-wise.

\subsubsection{Proof of \cref{thm:char:complete}.}
Observe that \cref{eq:rational_difference} is a rational function in $x$ of degree one in the numerator and two in the denominator. 
Under the conditions required in \cref{lemma:equivalence_on_line}, the function $L^{\policy,\policy'}_{P,\bar{s}}(x)$ is zero on a line segment.
In other words, the numerator in \cref{eq:rational_difference} is zero on a particular line segment.
Crucially, observe that a quadratic function is zero on a line segment \emph{only if the quadratic term cancels out}, \ie if $\bb'\bc-\bb\bc' = 0$.
This fact leads to the following theorem, which is our final characterization of \RBE policies. %

\completeness*

\begin{proof}
    Observe that the policy $\policy^\star$ satisfies one of the following three points:
    \begin{enumerate}
        \item $\hat\policySpace^\star$, $\hat\policySpace^\star_{(1)}$, or $\hat\policySpace^\star_{(2)}$ is a singleton, so that $\policy^\star$ is \RBE by \cref{corr:char:maxmin_unique} or \cref{corr:char:maxmax_unique};
        \item $\hat\policySpace^\star$, $\hat\policySpace^\star_{(1)}$, and $\hat\policySpace^\star_{(2)}$ are no singletons but either \cref{eq:completeness:deriv-max} or \cref{eq:completeness:deriv-min} holds with strict inequality, so that $\policy^\star$ is \RBE by \cref{corr:char:derivative};
        \item $\hat\policySpace^\star$, $\hat\policySpace^\star_{(1)}$, and $\hat\policySpace^\star_{(2)}$ are no singletons and \cref{eq:completeness:deriv} with non-strict inequality.
    \end{enumerate}
    We will prove \cref{thm:char:complete} by showing that, even in the third case, the policy $\policy^\star$ is \RBE.
    As $\policySpace^\star_{(2)}$ is not unique, there exists a policy $\policy' \in \policySpace^\star_{(2)} \setminus \{\policy^\star\}$ that, for all $\bar{s} \in S$, satisfies:
    \begin{align*}
        Z^{\policy^\star}_{P^{(1)},\bar{s}}(P^{(1)}_{\bar{s}}) &= Z^{\policy'}_{P^{(1)},\bar{s}}(P^{(1)}_{\bar{s}}), \\
        Z^{\policy^\star}_{P^{(2)},\bar{s}}(P^{(2)}_{\bar{s}}) &= Z^{\policy'}_{P^{(2)},\bar{s}}(P^{(2)}_{\bar{s}}), \\
        \nabla_{\mathbf{v}} Z^{\policy^\star}_{P^{(1)},\bar{s}}(P^{(1)}_{\bar{s}}) &= \nabla_{\mathbf{v}} Z^{\policy'}_{P^{(1)},\bar{s}}(P^{(1)}_{\bar{s}}) \\
        \nabla_{\mathbf{v}} Z^{\policy^\star}_{P^{(2)},\bar{s}}(P^{(2)}_{\bar{s}}) &= \nabla_{\mathbf{v}} Z^{\policy'}_{P^{(2)},\bar{s}}(P^{(2)}_{\bar{s}}).
    \end{align*}
    Hence, observe that the policies $\policy^\star$ and $\policy'$ satisfy the conditions in \cref{lemma:equivalence_on_line}.
    In addition, \cref{thm:char:complete} requires that, for all $\bar{s} \in S$, the line $g_{\bar{s}}(\lambda)$ between $P^{(1)}_{\bar{s}}$ and $P^{(2)}_{\bar{s}}$ either intersects the relative interior of $\cP_{\bar{s}}$, or completely covers $\cP_{\bar{s}}$.

    First, consider the case where $g_{\bar{s}}(\lambda)$ covers $\cP_{\bar{s}}$, \ie $\cP_{\bar{s}} \cap \{g_{\bar{s}}(\lambda)\} = \cP_{\bar{s}}$.
    In this case, \cref{lemma:equivalence_on_line} implies that the policies $\policy^\star$ and $\policy'$ have the same value in the entire uncertainty set.
    As a result, it holds that $\rho^{\policy'}_{P} = \rho^{\policy^\star}_{P}$ for all $P \in \cP$, so the policy $\policy^\star$ is \RBE.

    Second, consider the case where the line $g_{\bar{s}}(\lambda)$ intersects the relative interior of $\cP_{\bar{s}}$, that is, there exists $\lambda \in [0,1]$ such that $g_{\bar{s}}(\lambda) \in \relinterior{\cP_{\bar{s}}}$.
    We use the definition of $L^{\policy^\star,\policy'}_{P,\bar{s}}(x)$ in \cref{eq:rational_difference} to prove this case by contradiction.
    As such, suppose that the other policy $\policy'$ strictly dominates $\policy^\star$.
    In this case, there must exist a $P' \in \cP_{\bar{s}}$ where $\rho^{\policy'}_{P'} > \rho^{\policy^\star}_{P'}$.
    By \cref{lemma:equivalence_on_line}, the point $P'$ cannot be on the line segment between $P^{(1)}_{\bar{s}}$ and $P^{(2)}_{\bar{s}}$.
    Because the line $g_{\bar{s}}(\lambda)$ between $P^{(1)}_{\bar{s}}$ and $P^{(2)}_{\bar{s}}$ intersects the relative interior of $\cP_{\bar{s}}$, there also exists another point $P'' \in \cP$ such that the line through $P'$ and $P''$ is perpendicular to $g_{\bar{s}}(\lambda)$.
    Moreover, as the quadratic term of $L^{\policy^\star,\policy'}_{P,\bar{s}}(x)$ is zero (\ie it is a rational of degree one) and $L^{\policy^\star,\policy'}_{P,\bar{s}}(x)$ has a value of zero on $g_{\bar{s}}(\lambda)$, the fact that $\rho^{\policy'}_{P'} > \rho^{\policy^\star}_{P'}$ implies that $\rho^{\policy'}_{P''} < \rho^{\policy^\star}_{P''}$.
    In other words, the existence of a point $P'$ where $\policy'$ has \emph{higher} expected return than $\policy^\star$, implies the existence of another point $P''$ where $\policy'$ has \emph{lower} expected return than $\policy^\star$.
    Therefore, such a policy $\policy'$ that strictly dominates $\policy^\star$ cannot exist, so $\policy^\star$ is \RBE.
\end{proof}
A visualization of the proof of \cref{thm:char:complete} is given in \cref{fig:polytope}.

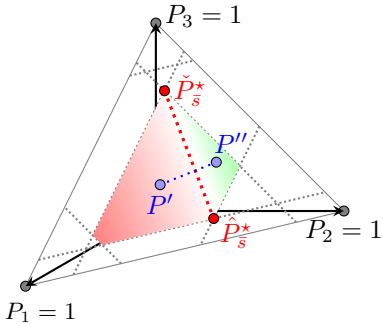
\begin{figure}[b!]
    \centering
        \newcommand\pgfmathsinandcos[3]{%
  \pgfmathsetmacro#1{sin(#3)}%
  \pgfmathsetmacro#2{cos(#3)}%
}

\usetikzlibrary{shadings} %

\begin{tikzpicture}[scale=2.8] 

    \pgfdeclarehorizontalshading{mygrad}{100bp}{
    color(0bp)=(red);
    color(10bp)=(red);
    color(51bp)=(white);
    color(70bp)=(green);
    color(100bp)=(green)
    }

    \pgfmathsetmacro\AngleFuite{150}
    \pgfmathsetmacro\coeffReduc{.8}
    \pgfmathsetmacro\clen{2}
    \pgfmathsinandcos\sint\cost{\AngleFuite}
    
    \begin{scope} [x     = {(\coeffReduc*\cost,-\coeffReduc*\sint)},
                   y     = {(1cm,0cm)}, 
                   z     = {(0cm,1cm)}]
        
        \newcommand\minA{0.15}
        \newcommand\maxA{0.7}
        \newcommand\minB{0.15}
        \newcommand\maxB{0.55}
        \newcommand\minC{0.1}
        \newcommand\maxC{0.8}

        \path coordinate (O) at (0,0,0)
              coordinate (A) at (1,0,0)
              coordinate (B) at (0,1,0)
              coordinate (C) at (0,0,1);
        \path coordinate (A_min_1) at (\minA,1-\minA,0)
              coordinate (A_min_2) at (\minA,0,1-\minA)
              coordinate (B_min_1) at (1-\minB,\minB,0)
              coordinate (B_min_2) at (0,\minB,1-\minB)
              coordinate (C_min_1) at (1-\minC,0,\minC)
              coordinate (C_min_2) at (0,1-\minC,\minC);
        \path coordinate (A_max_1) at (\maxA,1-\maxA,0)
              coordinate (A_max_2) at (\maxA,0,1-\maxA)
              coordinate (B_max_1) at (1-\maxB,\maxB,0)
              coordinate (B_max_2) at (0,\maxB,1-\maxB)
              coordinate (C_max_1) at (1-\maxC,0,\maxC)
              coordinate (C_max_2) at (0,1-\maxC,\maxC);
        
        \draw[gray] (A)--(B)--(C)--cycle;

        \draw[fill=gray] (A) circle (0.75pt) node[below, yshift=-0.1cm, xshift=0.2cm] {\small $P_{1}=1$};
        \draw[fill=gray] (B) circle (0.75pt) node[below] {$\small P_{2}=1$};
        \draw[fill=gray] (C) circle (0.75pt) node[right, yshift=0.1cm] {$\small P_{3}=1$};
        
        \draw[thick,-stealth,black] (O) -- (0,0,1);
        \draw[thick,-stealth,black] (O) -- (1,0,0);
        \draw[thick,-stealth,black] (O) -- (0,1,0);
        
        \draw[name path=Amin, densely dotted, thick, gray] (A_min_1) -- (A_min_2);
        \draw[name path=Bmin, densely dotted, thick, gray] (B_min_2) -- (B_min_1);
        \draw[name path=Cmin, densely dotted, thick, gray] (C_min_1) -- (C_min_2);
        
        \draw[name path=Amax, densely dotted, thick, gray] (A_max_1) -- (A_max_2);
        \draw[name path=Bmax, densely dotted, thick, gray] (B_max_2) -- (B_max_1);
        \draw[name path=Cmax, densely dotted, thick, gray] (C_max_1) -- (C_max_2);

        \begin{scope}
            \clip (A_min_1) -- (A_min_2) -- (A) -- cycle;
            \clip (A_max_1) -- (A_max_2) -- (C) -- (B) -- cycle;
            
            \clip (B_min_1) -- (B_min_2) -- (B) -- cycle;
            \clip (B_max_1) -- (B_max_2) -- (C) -- (A) -- cycle;
            
            \clip (C_min_1) -- (C_min_2) -- (C) -- cycle;
            \clip (C_max_1) -- (C_max_2) -- (B) -- (A) -- cycle;

            \shade[shading=mygrad, shading angle=21.5] (A)--(B)--(C)--cycle;
        \end{scope}

        \draw[fill=red] (0.15, 0.15, 0.7) circle (0.75pt) node[red, right] {$P^{(1)}_{\bar{s}}$};
        \draw[fill=red] (0.35, 0.55, 0.1) circle (0.75pt) node[red, below, xshift=0.3cm, yshift=0.15cm] {$P^{(2)}_{\bar{s}}$};
        \draw[name path=Amin, dotted, very thick, red] (0.15, 0.15, 0.7) -- (0.35, 0.55, 0.1);

        \draw[name path=Amin, dotted, thick, blue] (0.4, 0.3, 0.3) -- ($(0.4, 0.3, 0.3) - 1.5*(0.125, -0.1, -0.025)$);
        \draw[fill=blue!40] (0.4, 0.3, 0.3) circle (0.75pt) node[blue, below] {$P'$};
        \draw[fill=blue!40] ($(0.4, 0.3, 0.3) - 1.6*(0.125, -0.1, -0.025)$) circle (0.75pt) node[blue, above, xshift=0.2cm] {$P''$};
        
    \end{scope}

\end{tikzpicture}
    \caption{Visualization of the proof of \cref{thm:char:complete} for a convex polytopic uncertainty set $\cP_{\bar{s}}$ over three states. 
    The line segment between $P^{(1)}_{\bar{s}}$ and $P^{(2)}_{\bar{s}}$ is shown in red. 
    The color shade in the polytope depicts the difference $L^{\policy^\star,\policy'}_{P,\bar{s}}(x)$ in value between the policies $\policy^\star$ and $\policy'$.
    Red means $L^{\policy^\star,\policy'}_{P,\bar{s}}(x) < 0$, white means $L^{\policy^\star,\policy'}_{P,\bar{s}}(x) = 0$, and green means $L^{\policy^\star,\policy'}_{P,\bar{s}}(x) > 0$.
    Because $L^{\policy^\star,\policy'}_{P,\bar{s}}(x)$ is zero along the line segment and intersects the interior of the uncertainty set $\cP_{\bar{s}}$, for every point $P'$ where $L^{\policy^\star,\policy'}_{P,\bar{s}}(x) < 0$ (\ie $\policy'$ outperforms $\policy^\star$), there exists another point $P''$ where $L^{\policy^\star,\policy'}_{P,\bar{s}}(x) > 0$ (\ie $\policy'$ performs worse than $\policy^\star$).
    In particular, this point $P'' \in \cP_{\bar{s}}$ can be chosen to be any point such that the line through $P'$ and $P''$ is perpendicular to the line through $P^{(1)}_{\bar{s}}$ and $P^{(2)}_{\bar{s}}$.
    }   
\label{fig:polytope}
\end{figure}

\section{Details on Empirical Evaluation}
\label{appendix:models}
In this appendix, we provide further details about the models used in the empirical evaluation and the implementation of robust value iteration that we use.

\subsection{Gridworld Models}
We generate \emph{slippery gridworlds} of different sizes and with different numbers of obstacles, such as the instance shown in \cref{fig:gridworld_example}.
The objective for the agent is to minimize the expected number of steps to reach the target (in green) from the initial state (in blue).
Upon hitting an obstacle (in red), the agent resets to the initial state.

\paragraph{Interval MDP.}
For the interval MDP (IMDP) used in \cref{sec:experiment:prism}, we use the model depicted in \cref{fig:gridworld_imdp}.
For every direction (left, right, up, down), the agent can choose between two actions: one where the slipping probability $p$ is \emph{fixed}, and one where it belongs to the \emph{interval} $[q,p]$.
This model structure is repeated for every cell in the grid.

\paragraph{$s$-Rectangular RMDP.}
For the RMDP used in \cref{sec:experiment:rvi}, we use the model depicted in \cref{fig:gridworld_rmdp}.
Similar to the example RMDP in \cref{fig:rmdp1}, this gridworld RMDP has an $s$-rectangular uncertainty set, which is, in this case, parametrized by the maximum slipping probability $p$ and an improvement $q$.
The value of $p$ is fixed, \eg $p=0.25$, whereas the value of $q$ belongs to an interval, \eg $0 \leq q \leq 0.25$.
Thus, for $q = 0$, both action types yield the same value, whereas the \besteffort action dominates the non-\besteffort action for any $q > 0$.

\subsection{Robust Value Iteration}
In our experiments, we use two implementations of robust value iteration: one for IMDPs within the probabilistic model checker PRISM~\cite{DBLP:conf/cav/KwiatkowskaNP11}, and one for $s$-rectangular RMDPs that we implemented ourselves in Python.
Our own implementation of robust value iteration follows the standard form also described in \cref{sec:preliminaries}.
That is, given an initial policy $\policy \in \policySpace$ and uncertainty set $P \in \cP$, we iterate between the following steps:
\begin{enumerate}
    \item Given fixed $P$, for every state $s \in S$, update the policy $\policy(s)$ by maximizing the value $V(s)$ in state $s$:
    \begin{align*}    
        \policy(s) &\gets \argmax_{\policy(s) \in \dist{A}} \left\{ R^\policy(s) + \dotp{\gamma P^\policy(s)}{V} \right\},
        \\
        V(s) &\gets \max_{\policy(s) \in \dist{A}} \left\{ R^\policy(s) + \dotp{\gamma P^\policy(s)}{V} \right\}.
    \end{align*}
    \item Given fixed $\policy$ and $V$, for every state $s \in S$, update the worst-case transition function:\footnote{Here, we use $P(s,\cdot)$ to denote the transition probabilities in state $s \in S$ for all actions $a \in A$.}
    \begin{align*}
        P(s,\cdot) \gets \argmin_{ P(s,\cdot) \in \cP_s } \left\{ R^\policy(s) + \dotp{\gamma P^\policy(s)}{V} \right\},
    \end{align*}
    which we compute by solving a linear optimization program (under the assumption that $\cP_s$ is a convex polytope).
\end{enumerate}
For the gridworld experiments in \cref{sec:implementation}, the goal is to compute a policy that \emph{minimizes} the expected return.
Thus, for these experiments, we replace each $\max$ with $\min$ in the algorithm above and vice versa.

\begin{figure}[t!]
	\centering{
        \includegraphics[width=.54\linewidth]{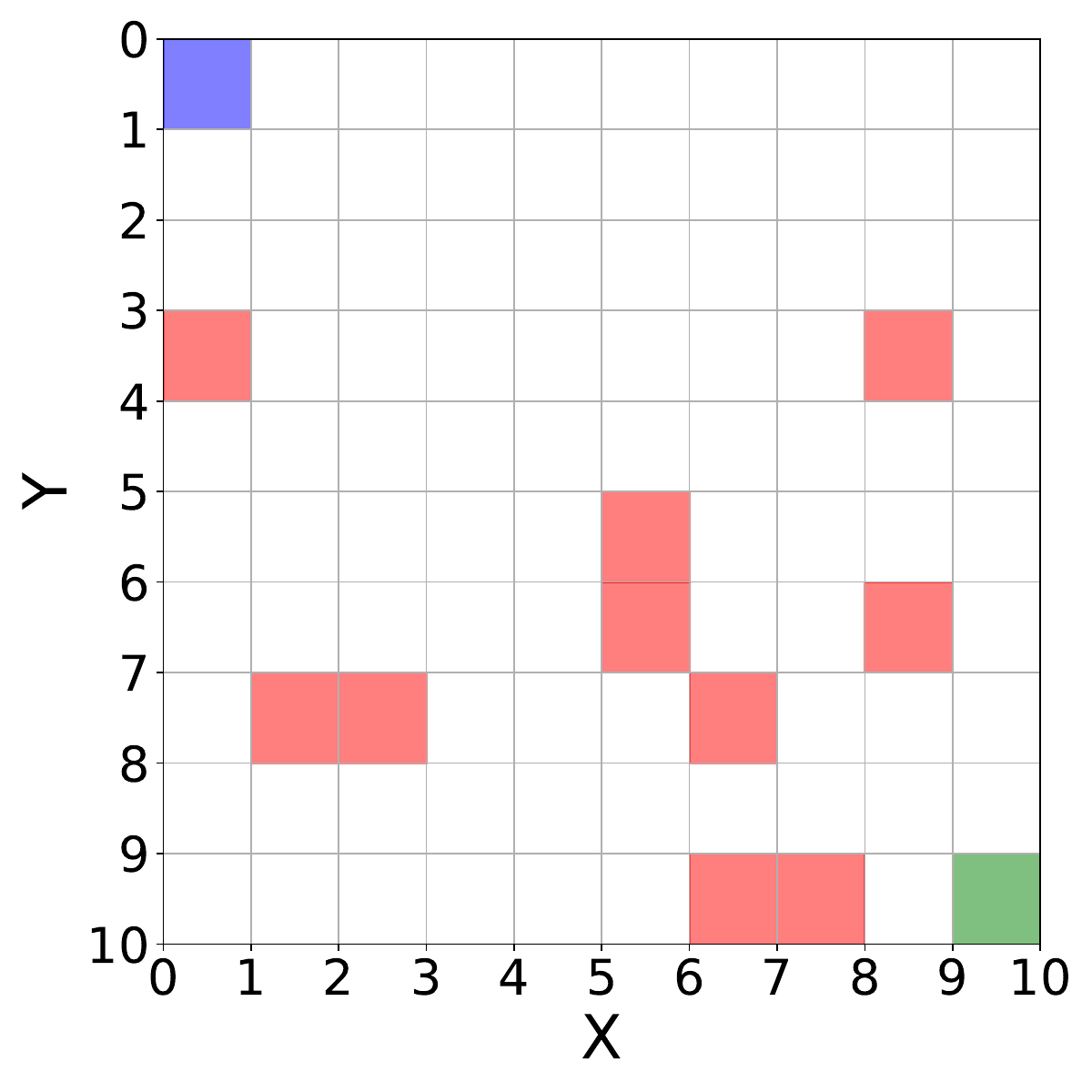}
        }
	\caption{Example instance of the slippery gridworld model of size $10 \times 10$ and with $10$ obstacles (in red), initial state (in blue), and target (in green). 
    Upon hitting a target, the agent is reset to the initial state.}
\label{fig:gridworld_example}
\end{figure}

\begin{figure}[t!]
    \centering
        \begin{tikzpicture}[
    state/.append style={inner sep=0pt, inner sep=0pt, minimum size=20pt}, 
    >=stealth,
    bobbel/.style={minimum size=1mm,inner sep=0pt,fill=black,circle},
    mynode/.style={rectangle,fill=white,anchor=center}]

    \node[state] (s00) at (1,0) {$0,0$};
    \node[state] (s10) at ($(s00) + (5,0)$) {$1,0$};

    \node[bobbel] (s00r) at ($(s00) + (2,0.4)$) {};
    \draw (s00) -- node[above]{right} (s00r);
    \draw (s00r) edge[->] node[above]{$1-p$} (s10);
    \draw (s00r) edge[->, bend right = 60, looseness=1.5] node[above]{$p$} (s00);

    \node[bobbel] (s00rBE) at ($(s00) + (2,-0.4)$) {};
    \draw (s00) -- node[below]{$\text{right}_{\BE}$} (s00rBE);
    \draw (s00rBE) edge[->] node[below, yshift=-0.1cm]{$[1-p,1-q]$} (s10);
    \draw (s00rBE) edge[->, bend left = 60, looseness=1.5] node[below]{$[q,p]$} (s00);

    \node (s01) [below=0.7cm of s00] {$\vdots$};
    \node (s11) [below=0.7cm of s10] {$\vdots$};
    \node (s20) [right=0.8cm of s10] {$\hdots$};
    \node (s21) [below=0.9cm of s20] {$\ddots$};
\end{tikzpicture}
    \caption{Structure of a single action in the slippery gridworld IMDP model used in \cref{sec:experiment:prism}. 
    For every move into one of the four cardinal directions (left, right, up, down), the agent has a \emph{normal} action with a fixed slipping probability $p$ (\eg $\text{right}$), and a \emph{\besteffort} action (\eg $\text{right}_{\BE}$) with a slipping probability interval $[q,p]$. 
    The structure shown is repeated for every cell in the grid.}
\label{fig:gridworld_imdp}
\end{figure}
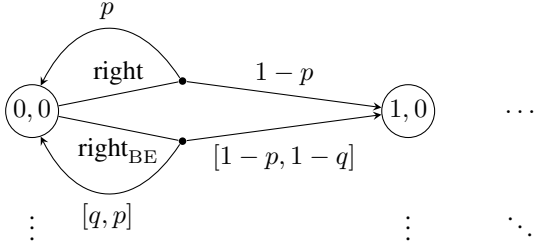

\begin{figure}[t!]
    \centering
        \begin{tikzpicture}[
    state/.append style={inner sep=0pt, inner sep=0pt, minimum size=20pt}, 
    >=stealth,
    bobbel/.style={minimum size=1mm,inner sep=0pt,fill=black,circle},
    mynode/.style={rectangle,fill=white,anchor=center}]

    \node[state] (s00) at (1,0) {$0,0$};
    \node[state] (s10) at ($(s00) + (5,0)$) {$1,0$};

    \node[bobbel] (s00r) at ($(s00) + (2,0.4)$) {};
    \draw (s00) -- node[above]{right} (s00r);
    \draw (s00r) edge[->] node[above]{$1-p+q$} (s10);
    \draw (s00r) edge[->, bend right = 60, looseness=1.5] node[above]{$p - q$} (s00);

    \node[bobbel] (s00rBE) at ($(s00) + (2,-0.4)$) {};
    \draw (s00) -- node[below]{$\text{right}_{\BE}$} (s00rBE);
    \draw (s00rBE) edge[->] node[below, yshift=-0.1cm]{$1-p+2q$} (s10);
    \draw (s00rBE) edge[->, bend left = 60, looseness=1.5] node[below]{$p - 2q$} (s00);

    \node (s01) [below=0.7cm of s00] {$\vdots$};
    \node (s11) [below=0.7cm of s10] {$\vdots$};
    \node (s20) [right=0.8cm of s10] {$\hdots$};
    \node (s21) [below=0.9cm of s20] {$\ddots$};
\end{tikzpicture}
    \caption{Structure of a single action in the slippery gridworld RMDP model used in \cref{sec:experiment:rvi}.
    This model has an $s$-rectangular uncertainty set, defined by a fixed worst-case slipping probability (\eg $p = 0.25$) and an uncertain improvement in the slipping probability of, \eg $0 \leq q \leq 0.1$.
    The structure shown is repeated for every cell in the grid.}
\label{fig:gridworld_rmdp}
\end{figure}
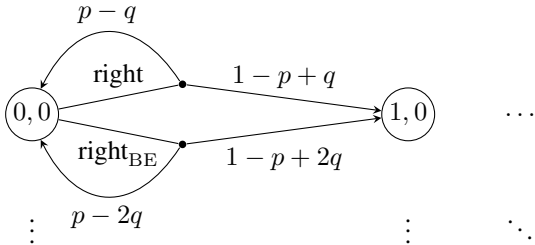

As discussed in \cref{sec:preliminaries}, randomized policies are necessary to achieve optimal values in $s$-rectangular RMDPs in general.
However, for the gridworld models, deterministic policies are sufficient for optimality, meaning that we may replace the (arg)max over distributions $\policy(s) \in \dist{A}$ in step 1 by an (arg)max over actions $a \in A$.

\paragraph{Parameters.}
In our implementation, we alternate between these steps until the value function has converged (up to a predefined $\epsilon > 0$, where we use $\epsilon = 10^{-4}$ in our experiments), or a predefined number of iterations is reached (we use a limit of $1\,000$ in our experiments).
In practice, this leads to a non-optimized but functional implementation of robust value iteration, which we use to compute policies and evaluate our methods for $s$-rectangular RMDPs.

\subsection{Results}
As described in \cref{sec:implementation}, we repeat every instance over 10 random seeds.
Note that the only source of randomness in our experiments comes from the generation of obstacles and the randomization in the order of defining the \besteffort versus non-\besteffort actions (reflected by the parameter $\nu$).
All standard deviations of the results presented in \cref{tab:gridworld_prism,tab:gridworld_rvi} are negligible and are thus omitted for clarity.

\end{document}